%% file: main_3dv.tex
\documentclass[10pt,twocolumn,letterpaper]{article}

\usepackage{3dv}
\usepackage{times}
\usepackage{epsfig}
\usepackage{graphicx}
\usepackage{amsmath}
\usepackage{amssymb}

\usepackage[pagebackref=true,breaklinks=true,colorlinks,bookmarks=false]{hyperref}

\usepackage{times}
\usepackage{epsfig}
\usepackage{graphicx}
\usepackage{amsmath}
\usepackage{amssymb}
\usepackage{amsthm}
\usepackage{algorithm} 
\usepackage{algpseudocode} 

\usepackage{array}  
\usepackage{caption}
\usepackage{subcaption}
\usepackage{calc}
\usepackage{verbatim} 
\usepackage{xcolor}
\usepackage{booktabs}

\usepackage{animate}

\usepackage{enumitem}
\setlist{topsep=0pt, leftmargin=*,noitemsep,topsep=0pt,parsep=0pt,partopsep=0pt}

\newtheorem{lemma}{Lemma}
\newtheorem{cond}{Condition}

\newcommand{\change}[1]{{#1}}

\renewcommand{\paragraph}[1]{\textbf{#1}}

\threedvfinalcopy %

\def\threedvPaperID{115} %

\usepackage{nopageno}
\begin{document}

\title{Shortest Paths in Graphs with Matrix-Valued Edges: \\
Concepts, Algorithm and Application to 3D Multi-Shape Analysis}

\author{Viktoria Ehm\\
TU Munich\\
\and
Daniel Cremers\\
TU Munich\\
\and
Florian Bernard\\
University of Bonn, TU Munich
}

\maketitle

\begin{abstract}
Finding shortest paths in a graph is relevant for numerous problems in computer vision and graphics, including image segmentation, shape matching, or the computation of geodesic distances on discrete surfaces. Traditionally, the concept of a shortest path is considered for graphs with scalar edge weights, which makes it possible to compute the length of a path by adding up the individual edge weights. Yet, graphs with scalar edge weights are severely limited in their expressivity, since oftentimes edges are used to encode significantly more complex interrelations. In this work we compensate for this modelling limitation and introduce the novel graph-theoretic concept of a shortest path in a graph with matrix-valued edges. To this end, we define a meaningful way for quantifying the path length for matrix-valued edges, and we propose a  simple yet effective algorithm to compute the respective shortest path. While our formalism is universal and thus applicable to a wide range of settings in vision, graphics and beyond, we focus on demonstrating its merits in the context of 3D multi-shape analysis.

\end{abstract}

\begin{figure}[htb]
\includegraphics[width=1\linewidth]{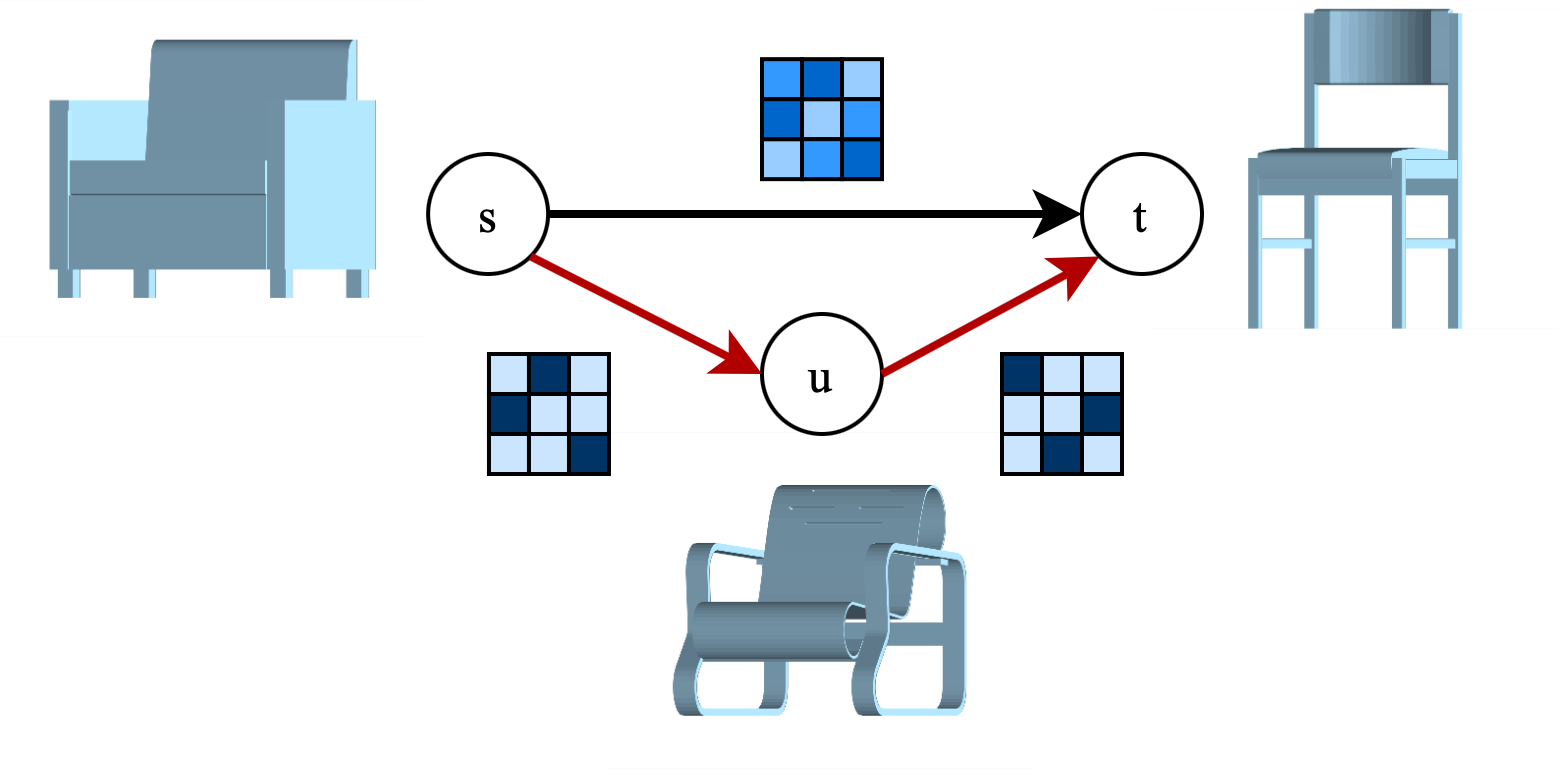}%
\vspace{-4mm}
\caption{\textbf{Shortest path problem in a graph with matrix-valued edges.} In this probabilistic multi-matching graph the nodes represent 3D shapes, and edges encode probabilistic correspondences between the shapes. Shapes $s$ and $t$ are more dissimilar compared to $s$ and $u$ (or  $u$ and $t$), so that the probabilistic correspondence of the path $(s,t)$ is more `fuzzy' compared to the composed path  $(s,u,t)$. 
Considering the `fuzzyness' of a path as its length, 
the composed path $(s,u,t)$ (red) is the shortest path between $s$ and $t$.
}
\label{fig:teaser}
\end{figure}

\section{Introduction}
Graphs  are a fundamental data structure for representing relations between entities. A graph comprises of nodes and edges, where the nodes represent the entities, and edges connect nodes in order to represent their relations. By associating a scalar weight to each edge, one can compute the length of a path along a sequence of connected nodes by summing up individual edge weights. The widely-used graph-theoretic problem 
of finding shortest paths in a graph has a high relevance for numerous tasks in
computer vision and graphics, 
including image segmentation~\cite{schoenemann2009combinatorial}, shape matching~\cite{laehner2016efficient}, computing similarities between shapes \cite{hilaga2001topology}, or the computation of geodesic distances on surfaces~\cite{mitchell1987discrete,kimmel1998computing}, among many others. While these are fundamentally different tasks, the respective shortest path formulations have in common that the weights associated to each edge in the graph are scalar-valued -- yet, many problems related to computer vision and graphics rely on graphs with more complex relations between nodes, such as synchronisation methods~\cite{arrigoni2020synchronization, pachauri2013solving}
, which are for example relevant for multi-shape analysis. In synchronisation problems, the nodes of the graph represent objects (\eg images, or 3D shapes), and the edges represent relative transformations between pairs of objects, such as permutation matrices for correspondence problems~\cite{bernard2019synchronisation},
or rigid-body transformations for alignment problems~\cite{bernard2015solution,arrigoni2016spectral}.

Although the example of synchronisation methods illustrates that modelling more complex relations is highly relevant for processing and analysing visual data, we believe that we currently lack the algorithmic machinery in order to fully exploit the potentials of such complex data. As a first step towards addressing this issue, in this work \emph{we conceptualise the shortest path problem for graphs with matrix-valued edges}. To this end, we address the fundamental question of how to -- in a meaningful way -- define the cost of a path comprising of matrix-valued edges, as well as how to find the shortest path between two nodes.
Experimentally, we demonstrate that such a formulation offers novel potentials in the context of 3D multi-shape analysis.
Overall, our main contributions are:
\begin{itemize}
    \item For the first time we  conceptualise the shortest path problem for graphs with matrix-valued edges.
    \item For computationally finding shortest paths we introduce an exact search algorithm, which is based on a simple pruning strategy that guarantees to find the globally optimal shortest path.
    \item As a proof of concept, we show that our novel formalism is beneficial in the context of 3D multi-shape analysis, \eg for finding intermediate shapes between pairs of shapes in a shape collection, to define a correspondence-free shape metric, or for shape morphing.
\end{itemize}

\section{Related Work}
In this section we summarise works that we consider most relevant to our approach.

\paragraph{Graphs in visual computing.}
Graphs are omnipresent in computer vision and graphics as they appear in an enormous variety of practical problems. For instance, Kurillo \etal \cite{kurillo2008wide} use graphs for multi-camera calibration, where the cameras are connected by edges whose weight represent the number of common points observed by pairs of cameras. 
In multiple object tracking,  graph nodes are used to represent possible spatial locations of objects in different frames, which are connected by edges that represent possible spatial trajectories~\cite{jiang2007linear}. 
For geometric analysis, Sundar \etal \cite{sundar2003skeleton} use skeleton graphs to compare 3D meshes using graph matching techniques.
 Wang \etal \cite{wang2018pixel2mesh} propose a method based on graph neural networks for constructing 3D meshes from single RGB images. %
 Overall, the diversity of these examples highlight the importance of graphs for addressing computational problems in visual computing.

\paragraph{Shortest path algorithms.}
Finding shortest paths in graphs is a popular graph-theoretic concept that is ubiquitous in many subdisciplines of computer science. Most commonly, a scalar weight is associated to each edge in the graph, so that the length of a path is computed by summing up the individual edge weights.
The famous Dijkstra algorithm \cite{dijkstra1959note} provides an efficient way to find shortest paths in graphs with positive edge weights.
The Bellman-Ford algorithm \cite{bellman1958routing,ford1956network} and the Floyd-Warshall algorithm \cite{floyd1962algorithm} also allow for negative edge weights under the assumption that there are no cycles with negative weight.
The A*-algorithm can improve the computational time in practice  based on a heuristic function~\cite{hart1968formal}. 

An alternative to summing up individual edge weights is the multiplication of edge weights. This is for example relevant for the `most reliable path' problem, where the edge weights represent failure probabilities, and the purpose is to find the path between two nodes that has the lowest probability of failure.
Roosta \etal \cite{roosta1982routing} and Pan \etal \cite{pan2008minimizing} address this problem by transforming the edge weights logarithmically, so that multiplications become summations, and the problem amounts to an ordinary shortest path problems with additive path costs. Petrovic \etal \cite{petrovic1979two} present  algorithms
where multiplications of the edge weights are used to determine the most reliable path. Similarly, Jain and Gopal \cite{jain1986algorithm} present an algorithm based on node removal techniques in order to reduce the computational effort compared to techniques based on Dijkstra or Floyd. While the mentioned works take into account multiplicative path lengths, they are only considered for the case of scalar edge weights. In contrast, in our work we consider matrix-valued edge weights, where we compute path lengths based on the composition of matrices (e.g.~via matrix multiplication).

\paragraph{Shortest paths in visual computing.}
The concept of shortest paths also plays an important role in computer vision and graphics. For example, 
exact \cite{chen1990shortest,kaneva2000implementation,mitchell1987discrete,kimmel1998computing} and approximate \cite{aleksandrov2005determining,lanthier1997approximating} approaches have been considered to determine the shortest path over a polyhedral surface.
Hilaga \etal \cite{hilaga2001topology} search for similarities between shapes by computing shortest paths to determine geodesic distances on the surface of every 3D shape.

Shortest paths have also been utilised for model-based 2D image segmentation ~\cite{schoenemann2009combinatorial}. Analogous approaches were proposed for 3D image segmentation, which, however, amounts to the significantly more difficult problem of finding a minimal surface 
\cite{grady2008minimal}.
Lähner \etal \cite{laehner2016efficient} use a shortest path formulation to determine a non-rigid 2D-to-3D shape matching, for which the product manifold between the 2D shape and the 3D shape forms the corresponding graph. %

Berclaz \etal \cite{berclaz2011multiple} use the shortest path algorithm for multiple object tracking, and Snavely \etal \cite{snavely2008skeletal} use a modified version of the Dijkstra  algorithm to determine structure from motion for large photo collections.  The latter appproach has led to  impressive results, like the digital rebuilding of Rome~\cite{agarwal2011building}.
Similar to our approach, Snavely \etal \cite{snavely2008skeletal} use matrices to describe similarities between nodes. However, to compute shortest paths they summarise each matrix into a scalar value by computing the matrix trace, so that eventually they arrive at shortest path problem with scalar edge weights.
While the discussed approaches emphasise that shortest paths play an important role in visual computing, existing works have the main limitation that they only consider scalar edge weights.

\paragraph{Graphs with matrix-valued edges and beyond.}
In order to represent complex interrelations between objects, oftentimes graphs with edges that go beyond simple scalar weights are utilised. One particular example are matrix-valued edges, which can for example represent relative spatial transformations, or correspondences between pairs of objects.
Such graphs with matrix-valued edges are common in synchronisation problems~\cite{arrigoni2020synchronization}, where nodes represent objects (images, cameras, shapes, etc.) and the edges represent relative transformations between them (rigid-body motion, correspondences, etc.). The purpose of synchronisation is to establish cycle consistency in the set of pairwise relative transformations.
One popular application of synchronisation is multi-view registration \cite{sharp2002multiview,martinec2007robust}, where edges model the rotation and translation between different views.  Thunberg \etal \cite{thunberg2017distributed} formulate the synchronisation problem over graphs with orthogonal matrices as a nonlinear least-squares problem. The algorithm was extended in \cite{thunberg2017distributedCdc} to also address translations. Similarly, Arrigoni \etal~\cite{arrigoni2016spectral} present a method for synchronising graphs with matrices in $SE(3)$, and Bernard \etal \cite{bernard2015solution} show a method to align objects that are connected by invertible linear transformations. \change{Synchronisation problems have also been utilised in various contexts for deep learning \cite{huang2019learning,huang2021multibodysync,gojcic2020learning, 
purkait2020neurora}.}
Lu \etal \cite{lu1997globally} model poses (nodes) and the relative measurement (edges) between them for posegraph optimisation. Similarly, Strasdat \etal \cite{strasdat2010scale} assign matrices 
to graph edges for the monocular SLAM problem.
Huang \etal \cite{huang2014functional} present a method to compute consistent functional maps in shape collections, where graph nodes represent functional vector spaces, and edges represent functional maps between them. 

Although there are numerous works that consider graphs with matrix-valued edges in a wide range of different contexts, to date the graph-theoretic concept of a shortest path has not yet been generalised towards matrix-valued edges.  The purpose of this work is to fill this gap by conceptualising shortest paths for graphs with matrix-valued edges.

\section{Concept of Matrix-Valued Shortest Paths}
\textbf{Notation.} We use upper-case letters for sets (calligraphic) and matrices (non-calligraphic), and lower-case letters for vectors and scalars. 

\textbf{Graphs with matrix-valued edges.} A graph is a tuple $\mathcal{G}=(\mathcal{V},\mathcal{E})$,
where $\mathcal{V}$ represents the nodes and $\mathcal{E} \subseteq \mathcal{V} \times \mathcal{V}$ the set of (directed) edges. 
To each edge $(u,v) \in \mathcal{E}$ between two nodes $u,v \in \mathcal{V}$ we associate a matrix $M_{uv} \in \mathcal{M}$, where %
we assume that the set $\mathcal{M}$  is closed under the (associative) composition operation $\circ: \mathcal{M} \times \mathcal{M} \rightarrow \mathcal{M}$ (e.g.~matrix multiplication or summation), and that it contains an identity element $I \in \mathcal{M}$ such that $M \circ I = M$  and $I \circ M = M$ for all $M \in \mathcal{M}$.
For all $u \in \mathcal{V}$, we set $M_{uu} = I$.

A path from a source node $s \in \mathcal{V}$ to a target node $t \in \mathcal{V}$ is a sequence of nodes that starts with $s$ and ends with $t$.
For the path $\pi = (s,u_1,\ldots,u_k,t)$ we define the \emph{composed edge matrix} $M_{\pi} = M_{(s,u_1,\ldots,u_k,t)} = M_{su_1} {\circ} M_{u_1u_2} {\circ} \ldots {\circ} M_{u_kt} \in \mathcal{M}$,~i.e.~we compose the matrices along the edges in the path via the operation $\circ$.

\textbf{Shortest path problem.} Let us consider the function $f: \mathcal{M} \rightarrow \mathbb{R}^+$ that maps a matrix to a non-negative real value.  Among all possible paths from node $s$ to node $t$, we denote the path $\pi^{\star}_{st}$ that gives rises to the lowest value of $f$ as \emph{shortest path}. We denote the composed matrices in the shortest path as $M_{\pi^{\star}_{st}}$, so that the corresponding cost is $f(M_{\pi^{\star}_{st}})$, which we write as $f(\pi^{\star}_{st})$ by abuse of notation.

\section{Algorithm}
\label{section:algorithm}

In the following we propose a search algorithm in order to find the shortest path between a source node $s \in \mathcal{V}$ and all other nodes $t \in \mathcal{V} \setminus{\{s\}}$. \change{
To this end, we introduce a generic framework that can handle arbitrary
cost functions.}  For the sake of simplicity we assume that the graph is complete,~i.e.~$\mathcal{E} = \mathcal{V} \times \mathcal{V}$.
To guarantee the correctness of our algorithm, we have to ensure that adding an edge to the current path can only increase the path cost (analogously to the Dijkstra algorithm, in which scalar-valued edges cannot have negative costs). Formally, we require that the function $f$ monotonically increases the path cost:
\begin{cond}[Monotonicity of path cost]\label{cond:mon}
The function $f$ monotonically increases the path cost in the sense that $f(M{\circ}X) \geq f(M)$
holds for all $M \in \mathcal{M},X \in \mathcal{M}$.
\end{cond}
In addition, we impose that $f(I) = 0$,~i.e.~the identity matrix  has zero path cost. Condition~\ref{cond:mon} gives rise to the following straightforward yet important result:

\begin{lemma}\label{lem:kIterations}
Consider $k \geq2$.
Let $\pi^{\star}_{sp}$ be the shortest path from $s$ to $p$ among all paths with exactly $k{-}1$ edges, and $\pi^{\star}_{st}$ the shortest path from  $s$ to $t$ with at most $k{-}1$ edges.  
If $f(\pi^{\star}_{st}) \leq f(\pi^{\star}_{sp})$, then any path $\pi$ that contains $p$ at the $k$-th position  cannot lead to a shorter path from $s$ to $t$ than $\pi^{\star}_{st}$.
\end{lemma}
\begin{proof}
Consider the subpath $\pi'$ of $\pi$ from $s$ to $p$. Since $\pi'$  comprises of $k{-}1$ edges, $f(\pi^{\star}_{st}) \leq f(\pi^{\star}_{sp}) \leq f(\pi')$, due to the optimality of $\pi^{\star}_{sp}$. Further, as  $\pi'$ is a subpath of $\pi$, from Condition~\ref{cond:mon} it follows that $f(\pi') \leq f(\pi)$, so that $f(\pi^{\star}_{st}) \leq f(\pi)$.
\end{proof}
This indicates that $f(\pi^{\star}_{sp})$ serves as lower bound for the cost of any path $\pi$ from $s$ to $t$ that has $p$ at the $k$-th position. 
Hence, %
we can prune  paths with $p$ at the $k$-th position from the search tree,
whenever their cost is larger than the shortest path from $s$ to $p$ with at most $k-1$ edges, as illustrated in Fig.~\ref{fig:sppruning}. To implement this, at a time we consider only paths comprising $k$ edges, for which we keep track of all candidates of intermediate nodes $p$ 
that can possibly lead to a shorter path between $s$ and $t$. Those candidates that can appear at the $k$-th position in the path from $s$ to $t$ are denoted by $T_{k}(t)$.
During our search, when considering paths from $s$ to $t$ comprising of exactly $k$ edges, we only consider paths of the form $\mathcal{A}(t) = \{ (s,u_2,...,u_{k},t) ~:~u_i \in \mathcal{T}_{i}(t), i = 2,...,k\}$.

\begin{figure}
    \centering
    \includegraphics[width=1\columnwidth]{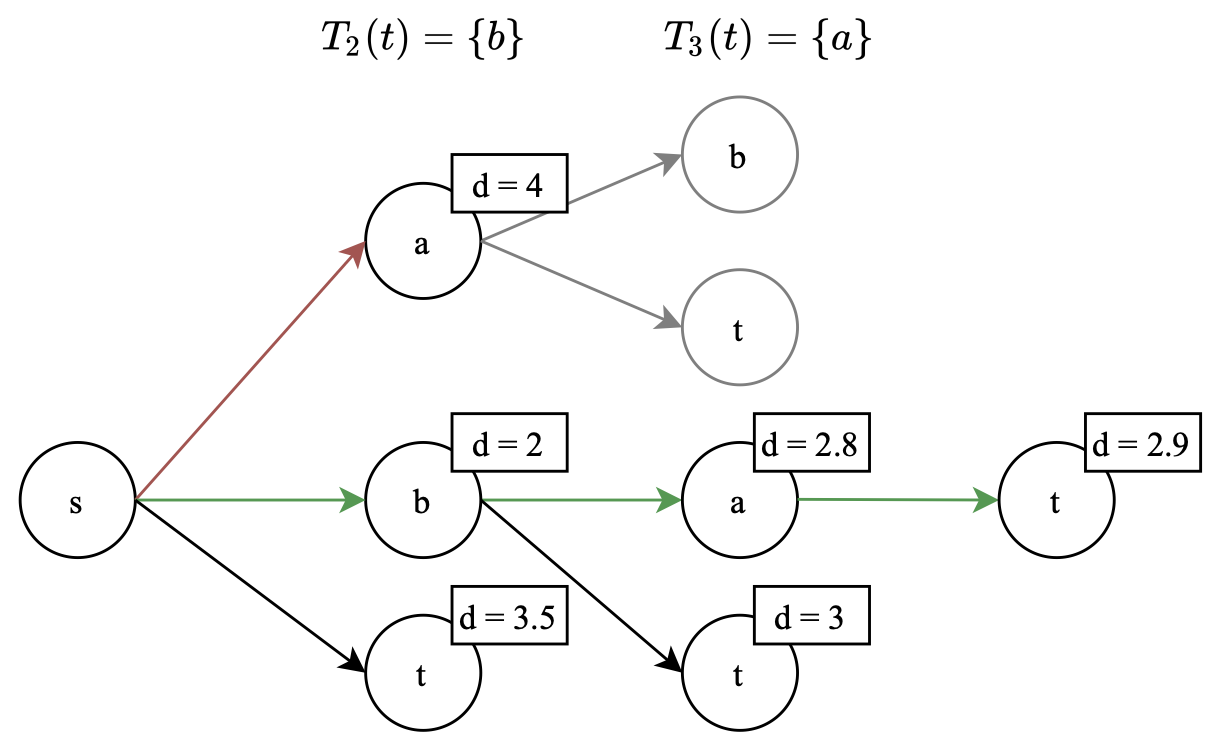}
    \vspace{-7mm}
    \caption{\textbf{Our algorithm prunes paths that cannot lead to a shorter paths.} 
    For finding a shortest path from $s$ to $t$, among the paths with  $k=2$ edges, only the node $b$ needs to be considered for the second position,~i.e.~$T_{2}(t) = \{b\}$.
   The direct path $\pi_{sa} = (s,a)$ is pruned since $f(\pi_{sa}) = 4$ is larger than the cost $f(\pi^{\star}_{st}) = 3.5$ of the shortest path from $s$ to $t$ with  $k<2$ edges. %
    For $k=3$, only the node $a$ needs to be considered for the third position,~i.e.~$T_{3}(t) = \{a\}$, because the path $\pi_{sba} = (s,b,a)$ has cost $f(\pi_{sba}) = 2.8$, which is smaller than the cost of the shortest path $f(\pi^{\star}_{sbt}) = 3$ from $s$ to $t$ with $k<3$ edges. %
    The globally optimal shortest path is marked in green.}
    \label{fig:sppruning}
\end{figure}
 Our algorithm is outlined in Alg.~\ref{alg:shortestpath}. It proceeds by
iterating over  $k=2,3\ldots$, where $k$ denotes the number of edges contained in a path.
For a given number of edges $k$, we determine all candidate paths that could potentially be a shortest path  between node $s$ and  target node $t$, and store them in the set $\mathcal{A}(t)$.
In case the set $\mathcal{A}(t)$ is empty for all $t$, we have evaluated all paths that are potential candidates for a shortest path, and we terminate.
If the cost $f(\pi)$ of a path  $\pi \in \mathcal{A}(t)$ from $s$ to $t$ is smaller than the cost $f(\pi^{\star}_{st})$ of the current shortest path $\pi^{\star}_{st}$ from $s$ to $t$,  $\pi^{\star}_{st}$ is updated. 

In the case of a complete graph, there are $\mathcal{O}(n!)$ candidates for the shortest path between source $s$ and a given target $t$, which need to be evaluated in the worst case.
Yet, due to Lemma~\ref{lem:kIterations}, we can reduce the average time complexity significantly in practice. 

\renewcommand{\algorithmicrequire}{\textbf{Input:}}
\renewcommand{\algorithmicensure}{\textbf{Output:}}

\algnewcommand{\Inputs}[3]{%
  \State \textbf{Input:}
  \State #1
  \State #2
  \State #3
}

\algnewcommand{\Initialise}[2]{%
  \State \textbf{Initialise:}
  \Statex \hspace*{\algorithmicindent}\parbox[t]{.8\linewidth}{\raggedright #1}
  \Statex \hspace*{\algorithmicindent}\parbox[t]{.8\linewidth}{\raggedright #2}
}

\begin{algorithm}
\begin{algorithmic}[1]
\caption{Matrix-valued shortest path algorithm
}\label{alg:shortestpath}
\State \textbf{Input:}
\State $\mathcal{V}$ (set of all nodes)
\State $s \in \mathcal{V}$ (source node)
\State $\{M_{uv}\}$ (edge matrices between $u$ and $v$)
\State \textbf{Initialise:}
\State  $\mathcal{T} = \mathcal{V} \setminus{\{s\}}$ (set of all target nodes)
\State $\pi^{\star}_{st} = (s,t)$
(direct path between $s$ and $t$)
\For {$k= 2...|\mathcal{V}|$}
\State $\mathcal{T}_{k}(t) = \{ p {\in} \mathcal{T} \setminus{\{t\}}: f({\pi^{\star}_{sp}}) {<} f({\pi^{\star}_{st}}) \}$  
\State $\mathcal{A}(t) = \{ (s,u_2,...,u_{k},t) :u_i \in \mathcal{T}_{i}(t), i = 2,...,k\}$

\If{$\mathcal{A}(t) = \emptyset$ for all $t \in \mathcal{T}$}
\Return
\EndIf

\For{ $t \in \mathcal{T}$ and $\pi \in \mathcal{A}(t)$}
\If {$f({\pi}) <  f({\pi^{\star}_{st}})$}
\State $\pi^{\star}_{st} = \pi$
\EndIf
\EndFor
\EndFor
\end{algorithmic}
\end{algorithm}

\section{Application to Multi-Shape Analysis}
\input{./figures_inputs/closest_shapes_smal}
We demonstrate the merits of our matrix-valued shortest path formalism \change{and its universal applicability to a diverse range of} tasks in multi-shape analysis. To this end, as a proof of concept we consider its application for the definition of a correspondence-free shape metric, for finding intermediate shapes between a given pair of shapes, and for shape morphing. 

For the rest of the paper, we consider the specific choice $\mathcal{M} = \{ X \in [0,1]^{n \times n}~:~X \mathbf{1} = \mathbf{1}, X^T \mathbf{1} = \mathbf{1}\}$,~i.e.~the set of doubly-stochastic matrices, and we use matrix multiplication as composition operation $\circ$. Since the set of doubly-stochastic matrices forms a semigroup with matrix multiplication as group operation~\cite{farahat1966semigroup}, $\mathcal{M}$ is closed under matrix multiplication, and contains the identity matrix $I$ as identity element.
The doubly-stochastic matrix $M_{xy}$ connecting node $x$ to $y$ can be seen as a probabilistic correspondence between $n$ parts on $x$ and $n$ parts on $y$. Hence, we define all edge matrices in a symmetric manner, so that $M_{yx} = M_{xy}^T$.
In order to quantify the cost of a path, we consider the `fuzzyness' of the composed probabilistic correspondence along a path. To this end, we set $f$ to the total entropy $H: [0,1]^{n \times n} \rightarrow  \mathbb{R}^+$,~i.e. %
\begin{equation}
H(X) = -\sum_i \sum_j X_{ij} {\cdot} \log(X_{ij}),
\end{equation}
\change{which is commonly used for this purpose~\cite{souiai2015entropy}}. We use the common convention that $0 {\cdot}\log(0) = 0$.

The following confirms that the function $H$ monotonically increases the path cost (as required in Condition~\ref{cond:mon}):
\begin{lemma}
For $M,X$ being doubly-stochastic matrices it holds that  %
$H(MX) \geq H(M)$.
\end{lemma}
\begin{proof}
The doubly stochastic matrix $X$ can be expressed as a convex combination of permutation matrices (Birkhoff's theorem~\cite{birkhoff1946three}),~i.e.~$X = \sum_{i} \alpha_i P_i$, where $P_i$ is a permutation matrix and $\alpha_i \geq 0$ with $\sum_i \alpha_i = 1$. Moreover, the function $H(\cdot)$ is concave over the domain $[0,1]^{n \times n}$~\cite{cover1999elements}, so that  $H(MX) = H(M \sum_{i} \alpha_i P_i)  \geq \sum_i \alpha_i H(M P_i) = \sum_i \alpha_i H(M) = H(M)$, where the second-last equality follows since reordering the columns of the matrix $M$ by $P_i$ does not affect the value of $H(\cdot)$. \change{Note that for the specific case of $X$ being binary, the inequality is tight,~i.e.~$H(MX) = H(M)$.}
\end{proof}

\subsection{Probabilistic Multi-Matching Graph}
Next, we explain how we construct our probabilistic multi-matching graph for a given collection $\mathcal{S}$ of 3D shapes.
The $(n {\times} n)$-dimensional
doubly-stochastic matrices $M_{xy}$ represents probabilistic correspondences between shapes $x \in \mathcal{S}$ and $y \in \mathcal{S}$. \change{In contrast to commonly-used scalar edges, such matrix-valued edges constitute a powerful representation that can capture  fine-grained information about correspondences between individual points of shapes}. To obtain them, we first  compute SHOT features for all vertices of each 3D shape \cite{tombari2010unique}. 
Subsequently, in order to allow for shapes with a varying number of vertices, we  cluster all vertices of one shape into $n$ clusters 
using the k-means algorithm~\cite{lloyd1982least}, where we choose $n=28$ in all experiments.
Then,
we encode the SHOT feature distribution of all vertices within a cluster in terms of percentile statistics, which serves as a cluster-specific feature descriptor that summarises the characteristics of each cluster.
Doing so allows us to compute the similarity matrix between the clusters of two 3D shapes, for which we consider a Gaussian kernel applied to the $\ell_2$-norm of the difference between cluster feature descriptors, stored as $n \times n$ matrix.
In order to obtain the doubly-stochastic matrix $M_{xy}$, we use the
Sinkhorn matrix scaling algorithm \cite{sinkhorn1967concerning}. Thus, each entry in $M_{xy}$ can be interpreted as the probability to match a cluster of  shape $x$ to a cluster of  shape $y$. Overall, the probabilistic multi-matching graph $\mathcal{G} = (\mathcal{S}, \mathcal{E})$ is obtained by identifying each shape in $\mathcal{S}$ with a node, where all pairs of nodes are connected by an edge to which the respective probabilistic correspondence matrix is assigned.
Further details regarding the probabilistic multi-matching graph generation are presented in the Appendix.

\subsection{Correspondence-Free Shape Metric}

In this experiment we use our algorithm to define a metric on a given dataset $\mathcal{S}$ of 3D shapes. The distance  between the shape $x \in \mathcal{S}$ and the shape $y \in \mathcal{S}$ is defined as the total entropy of the shortest path between them,~i.e.~$H(M_{\pi_{xy}^\star})$, which we also write as $H(\pi_{xy}^\star)$ by abuse of notation. The following holds for $H$:
\begin{lemma}
For our probabilistic multi-matching graph $\mathcal{G} = (\mathcal{S}, \mathcal{E})$, the function $H$ fulfils to following properties:
\begin{enumerate}
    \item \textbf{Identity.} For all $x \in \mathcal{S}, H(\pi_{xx}^\star) = H(M_{xx}) = 0$. 
    \item \textbf{Non-negativity.} For all $x, y \in \mathcal{S}, H(\pi_{xy}^\star) \geq 0$.
    \item \textbf{Symmetry.} For all $ x,y \in \mathcal{S}$, $H(\pi_{xy}^\star)=H(\pi_{yx}^\star)$.
\item \textbf{Triangle Inequality (under $\circ$).} For all $x, y, z \in \mathcal{S}, H(\pi_{xz}^\star) \leq H(\pi_{xy}^\star \circ  \pi_{yz}^\star)$, where in this context the operation $\circ$ denotes the concatenation of two paths.
\end{enumerate}
\end{lemma}
\begin{proof}
1. By definition, $M_{xx}$ is the identity matrix $I$, and $H(I) = 0$. Since there cannot be a shorter path with lower cost (due to 2.), the corresponding path $(x)$ constitutes a shortest path from $x$ to $x$.\\
2.  The total entropy  maps doubly stochastic matrices to non-negative numbers,  
so $H(M) \geq 0$ for any $M \in \mathcal{M}$.\\
3. The commutativity of summation implies that the total entropy is symmetric in the sense that $H(X) = H(X^T)$ for any $X \in \mathcal{M}$. Since by definition $M_{xy} = M_{yx}^T$ for all $x,y \in \mathcal{S}$, the shortest path from $x$ to $y$ must be the same as the shortest path from $y$ to $x$ in reverse order. As such,
$H(\pi_{xy}^\star) = H(M_{\pi_{xy}^\star}) = H(M_{xu_1}  M_{u_1u_2}  ... M_{u_k y}) = H(M_{u_k y}^T ...  M_{u_1u_2}^T  M_{xu_1}^T ) = H(M_{yu_k} ...  M_{u_2u_1}  M_{u_1x} ) = H(M_{\pi_{yx}^\star}) = H(\pi_{yx}^\star)$.\\
4. Since $H(\pi_{xz}^\star)$ is the shortest path between node $x$ and node $z$, any other path between $x$ and $z$ via $y$ cannot lead to a smaller cost.
\end{proof}

\paragraph{Results.}
First, we run our algorithm to determine the distance between all pairs of 3D shapes. Given these distances, we obtain the $k$ nearest neighbour shapes for a given query shape. 
In this experiment we consider the 41 toy animals of the SMAL dataset \cite{Zuffi:CVPR:2017}. %
The dataset is categorised into five animal families (cats, dogs, cows, horses, hippos). For our evaluation, we interpret a nearest neighbour shape that belongs to the same animal family like the query shape as correct. We compare our metric with two baselines for computing nearest neighbours:
\begin{itemize}
    \item The Euclidean distance (Eucl) between the vertices of shapes. In that case, vertex-to-vertex correspondences between the shapes need to be known. We emphasise that 
    finding correspondences between a shape collection~\cite{Bernard_2019_ICCV} is an extremely difficult problem, and thus this setting is rather unrealistic in many practical scenarios.
    \item Hence, we additionally consider the Euclidean distance after using the iterative closest point (ICP) algorithm \cite{besl1992method, chen1992object} for registering the randomly rotated shapes without known correspondences. This is a more realistic setting, since here we assume that we do neither know correspondences, nor the spatial alignment between shapes.
\end{itemize}
\input{figures_inputs/errorplot}
\input{figures_inputs/sp_k_3}
 In Fig.~\ref{shapeRetrievalQualitativ} we show the three closest shapes of four different animals. It can be seen that our algorithm leads to the best results,~i.e.~it finds the most neighbour shapes from the same animal family.
To quantitatively evaluate the performance of the three methods, we consider he proportion of correct nearest neighbours for the $i$-th animal and a given number of neighbours $k$, denoted as $g_{ik} = \frac{\text{TP}(i,k)}{\text{min}(m_i, k)}$.
The value of $g_{ik}$ is evaluated for each of the three methods, where $\text{TP}(i,k)$ denotes  the total number of true positives of the $i$-th animal among the $k$ nearest neighbour shapes obtained by the respective method.
 Further, 
 $m_i$ is the number of animals that belong to the category of the $i$-th animal, and $k$ denotes the number of nearest neighbours that are considered. 
 Fig.~\ref{fig:shapeRetrievalQuantitativ} shows that our method significantly outperforms both the Euclidean distance and ICP.

\subsection{Intermediate Shapes} 
\input{figures_inputs/interm_tosca}
\input{figures_inputs/chair_good}
\input{figures_inputs/morphing}
In this experiment we consider the task of finding intermediate shapes between two given shapes. This  is closely related to shape interpolation~\cite{kilian2007geometric, heeren2012time, eisenberger2019divergence, eisenberger2020hamiltonian}, which is typically based on the assumption that there exists a smooth deformation between two given shapes.
The difference between an intermediate shape and shape interpolation is 
that intermediate shapes must be part of the shape collection. Hence, while interpolation can lead to implausible shapes, this cannot happen for intermediate shapes. %
Moreover, finding  intermediate shapes within our shortest path setting does not require vertex-to-vertex correspondences between shapes, which is often necessary for interpolation methods that rely on explicit shape representations.

\change{
We consider two settings for this experiment: }
(i) In the first setting, we find the shortest path between two given shapes that has a fixed number of edges $k$. By doing so, we can exactly specify the number of desired intermediate shapes, so that this setting is relevant whenever we are interested in a fixed number of intermediate shapes. Results for shortest paths among all paths with a fixed number of edges are shown in Fig.~\ref{fig:sp_k_3}.
(ii) In the second setting, we consider the shortest path among all possible paths, rather than paths with a fixed number of edges. Here, the number of obtained intermediate shapes is directly determined by the dataset. As such, there is also the possibility that the direct path leads to the lowest cost, and thus there may not always exist intermediate shapes between a pair of shapes. Respective results are shown in Figs.~\ref{fig:intermediateTosca} and~\ref{fig:goodChairs}. We can even find reasonable intermediate shapes when 
the topologies between shapes vary drastically, as demonstrated in Fig.~\ref{fig:goodChairs}.

\subsection{Shape Morphing}
A use case for intermediate shapes between two shapes is shape morphing. To do so, given a source and a target shape, we first find all intermediate shapes along the shortest path between the source and the target. Subsequently, we perform a piecewise linear interpolation between pairs of shapes along this shortest path. In contrast, a naive linear interpolation between the source and target leads to severe artefacts, which are particularly prominent if both shapes differ significantly. 
We illustrate results in Fig.~\ref{fig:shapeMorph}  for morphing two human body shapes with large poses differences. Animated results can be found in the Appendix. \change{Although more involved interpolation methods could also be used (both in conjunction with our method, and for  direct source-to-target morphing), here we use a simple linear interpolation to highlight potential problems due to a direct morphing, which is effectively circumvented using our method.}

\section{Discussion \& Future Work}
While we believe that our work offers a strong potential %
for numerous applications, even beyond multi-shape analysis, interesting open research questions are remaining. %

\begin{table}[]
    \centering
    \resizebox{\linewidth}{!}{%
       \begin{tabular}{lccc}
       \toprule
        \textbf{Dataset} & \textbf{\#Shapes} &
        \textbf{Time (SP) [s]}
        &\textbf{Time (CERT) [s]} \\
        \midrule
        SMAL & 41 & 1.6 & 154 \\
        ShapeNet Chairs & 51 & 3.8& 35.2 \\
        TOSCA (`Michael') & 20 & 0.33 & 1.6 \\
        Non-Rigid World & 78 & 147.45$^{\dagger}$ & 147.45$^{\dagger}$ \\
        \bottomrule 
        \end{tabular}
    }
        \vspace{-3mm}
    \caption{Runtimes of our algorithm for finding the globally optimal shortest paths between all node pairs (SP), and for certifying their global optimality (CERT).}
    \label{tab:runtime}
\end{table}
\textbf{Computational cost.} Our proposed shortest path problem does (in general) not exhibit the optimal substructure property,~i.e.~a subpath of a shortest path may not necessarily be a shortest path. As such, generalising existing shortest paths paradigms (e.g.~Dijkstra) to our setting would generally not guarantee the global optimality of so-obtained `shortest' paths. While Alg.~\ref{alg:shortestpath}  finds globally optimal shortest paths, it has an exponential worst-case time complexity. For moderately-sized graphs (${\leq}51$ nodes), we can find shortest paths between all nodes within seconds, and we certify the global optimality within minutes (see Table~\ref{tab:runtime}). $^{\dagger}$~For the larger Non-Rigid World dataset~\cite{bronstein2006efficient}, where we included quadrupeds, the centaur and the Michael shape, we limit the maximal number of edges in a path to $k_{\text{max}} {=} 3$.

\textbf{Edge matrix construction.}  In our experiments we  considered edge matrices obtained via  percentile statistics summarising the SHOT feature distribution of vertex clusters. While we demonstrated various intriguing results,
we found that e.g.~inconsistent clusters may cause problems (see Appendix). Yet, this is neither a limitation of our shortest path formalism, nor of our  algorithm, but rather stems from the specific choice of edge matrices.
The exploration of improved graph construction methods is an interesting direction for future work. One way may be the integration of our method into end-to-end trainable deep neural networks,~e.g.~via differentiable programming, in order to learn task-specific optimal edge matrices directly from data.
\change{Additionally we expect that our method
will inspire follow-up works that consider shortest paths for
edges with even more elaborate attributes than matrices.}

\section{Conclusion}
For the first time we have conceptualised the shortest path problem for graphs with matrix-valued edges. To this end, we introduced a generic modelling framework, along with a globally optimal search strategy for finding respective shortest paths. As a proof of concept, we studied matrix-valued shortest paths in the context of multi-shape analysis.
We believe that our formalism has a strong potential for numerous other applications in vision, graphics, and beyond.
Overall, we hope that our work will spark interest in developing more involved algorithmic machinery that is better capable of modelling and analysing the complex interrelations that occur in real-world visual computing problems.

{\small
\bibliographystyle{ieee_fullname}
\bibliography{egbib}
}

\clearpage 
\appendix

\input{./supplementary_material}

\end{document}

%% file: figures_inputs/closest_shapes_smal.tex
\newcommand{\includeSMAL}[3]{\includegraphics[width= 1\linewidth] {figures/top_list_#1/#2_#3.png}}

\newcommand{\includeSMALQuery}[3]{\includegraphics[width= 0.12\linewidth] {figures/top_list_#1/#2_#3.png}}

\newcommand{\includeSMALRow}[5]{
\includeSMAL{#3}{#1}{#2_2} & \includeSMAL{#4}{#1}{#2_3} & \includeSMAL{#5}{#1}{#2_4}
}

\begin{figure*}[h]
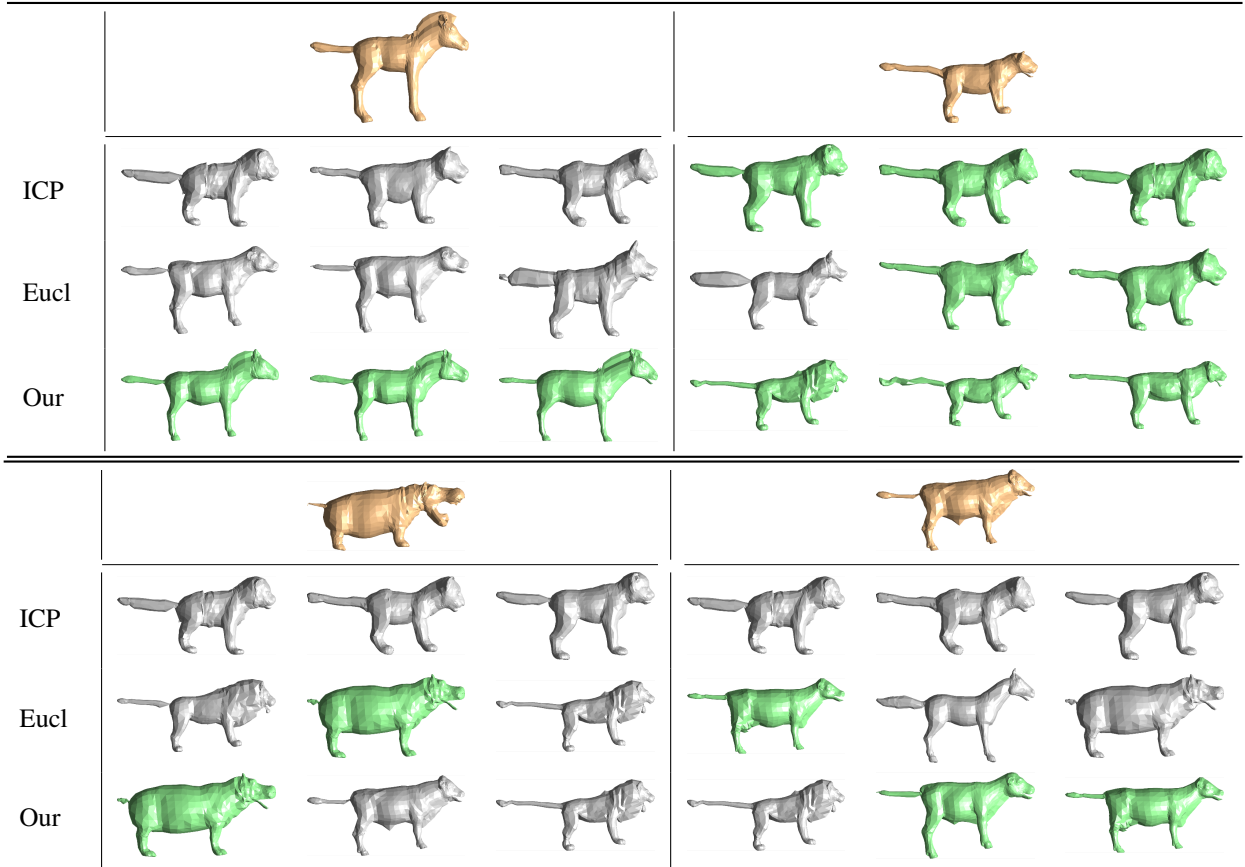
 \centering
\begin{tabular}{m{0.05\linewidth}|  m{0.12\linewidth}  m{0.12\linewidth}  m{0.12\linewidth} | m{0.12\linewidth}  m{0.12\linewidth}  m{0.12\linewidth} } 
\toprule

& \multicolumn{3}{c|}{\includeSMALQuery{ini}{sp}{3_1}} &  \multicolumn{3}{c}{\includeSMALQuery{ini}{sp}{5_1}}\\\cmidrule(r){2-4}\cmidrule(l){5-7}

ICP & \includeSMALRow{icp}{3}{gray}{gray}{gray} & 
\includeSMALRow{icp}{5}{green}{green}{green}\\

Eucl & \includeSMALRow{eucl}{3}{gray}{gray}{gray} & 
\includeSMALRow{eucl}{5}{gray}{green}{green}\\

Our & \includeSMALRow{sp}{3}{green}{green}{green} & 
\includeSMALRow{sp}{5}{green}{green}{green}\\

\bottomrule
\end{tabular}

\begin{tabular}
{m{0.05\linewidth}|  m{0.12\linewidth}  m{0.12\linewidth}  m{0.12\linewidth} | m{0.12\linewidth}  m{0.12\linewidth}  m{0.12\linewidth} } 
\toprule
& \multicolumn{3}{c|}{\includeSMALQuery{ini}{sp}{1_1}} &  \multicolumn{3}{c}{\includeSMALQuery{ini}{sp}{2_1}}\\\cmidrule(r){2-4}\cmidrule(l){5-7}

ICP & \includeSMALRow{icp}{1}{gray}{gray}{gray} & 
\includeSMALRow{icp}{2}{gray}{gray}{gray}\\

Eucl & \includeSMALRow{eucl}{1}{gray}{green}{gray} & 
\includeSMALRow{eucl}{2}{green}{gray}{gray}\\

Our & \includeSMALRow{sp}{1}{green}{gray}{gray} & 
\includeSMALRow{sp}{2}{gray}{green}{green} \\
\bottomrule
\end{tabular}
\vspace{-3mm}
\caption{\textbf{Closest shapes of SMAL toys  \cite{Zuffi:CVPR:2017}.} The three closest shapes (left to right) of the four query shapes (gold) are shown. Shapes that belong to the same family as the query shape are shown in green. Our algorithm  finds the most shapes from the same family. Note that each shape is scaled individually for visualisation purposes.}
\label{shapeRetrievalQualitativ}
\end{figure*}

%% file: figures_inputs/errorplot.tex
 
 \begin{figure}
    \centering
    \includegraphics[width=1\columnwidth]{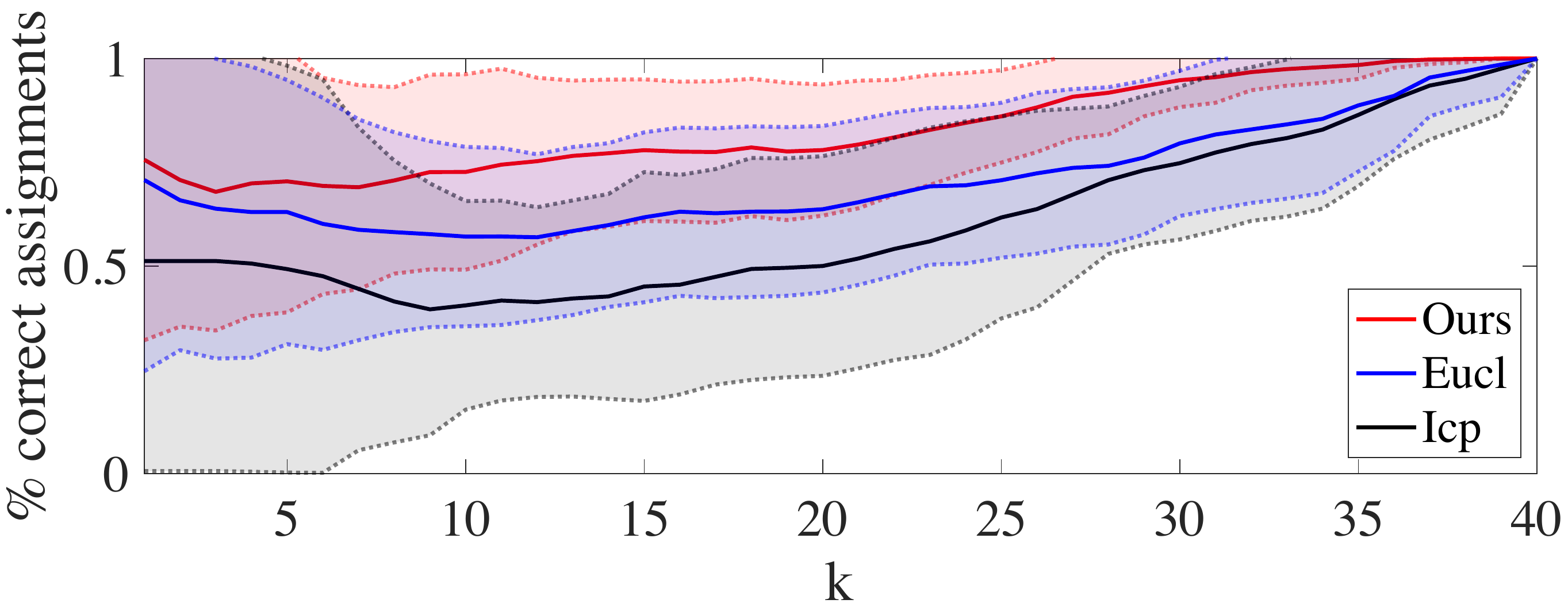}
    \vspace{-7mm}
    \caption{{Comparison of the \textbf{proportion of correct nearest neighbours} obtained by ICP, the Euclidean distance and our approach}. The number of nearest neighbours $k$ varies along the horizontal axis, and the vertical axis shows the mean (solid lines) and the standard deviation (filled area and dashed lines) of the $g_{ik}$ over all animals $i$ for a given $k$. 
    Our approach clearly outperforms the others.}
    \label{fig:shapeRetrievalQuantitativ}
\end{figure}

%% file: figures_inputs/sp_k_3.tex
\newcommand{\figWidth}{0.205\linewidth}

\begin{figure*}[h!]
    \centering
    
         \begin{subfigure}[b]{\figWidth}
    \includegraphics[width= 1\linewidth] {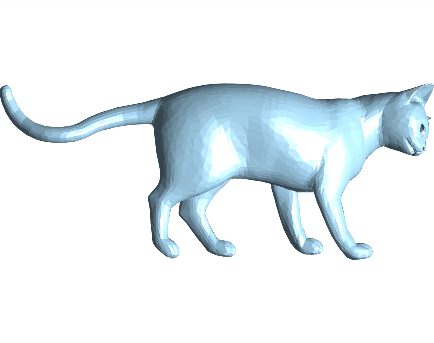}
\end{subfigure}\hfill
\begin{subfigure}[b]{\figWidth}
    \includegraphics[width= 1\linewidth] {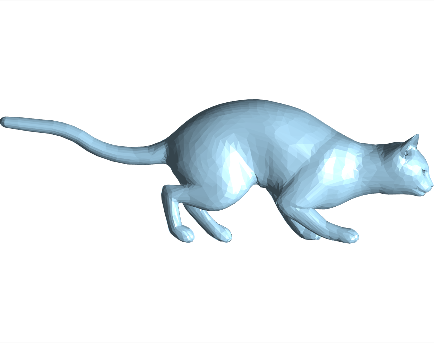}

\end{subfigure}\hfill
\begin{subfigure}[b]{\figWidth}
    \includegraphics[width= 1\linewidth] {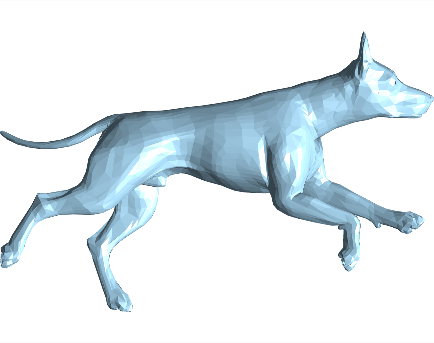}
\end{subfigure}\hfill
\begin{subfigure}[b]{\figWidth}
    \includegraphics[width= 1\linewidth] {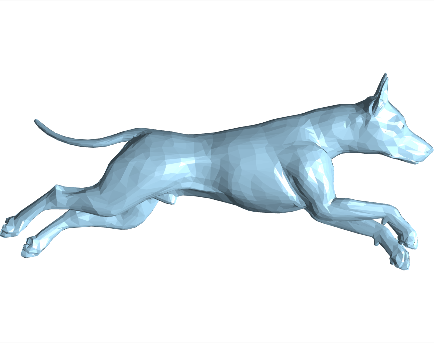}
\end{subfigure}

\vspace{-3mm}

        \begin{subfigure}[b]{\figWidth}
    \includegraphics[width= 1\linewidth] {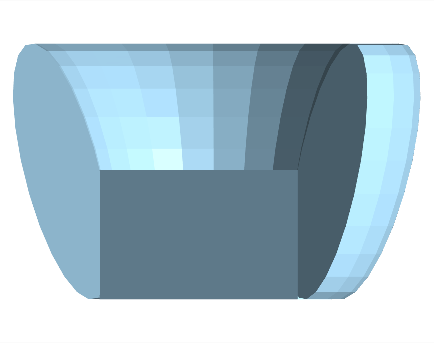}
\end{subfigure}\hfill
\begin{subfigure}[b]{\figWidth}
    \includegraphics[width= 1\linewidth] {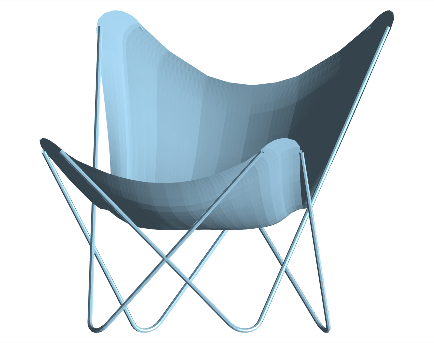}

\end{subfigure}\hfill
\begin{subfigure}[b]{\figWidth}
    \includegraphics[width= 1\linewidth] {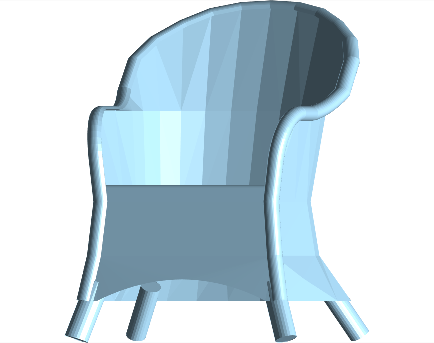}
\end{subfigure}\hfill
\begin{subfigure}[b]{\figWidth}
    \includegraphics[width= 1\linewidth] {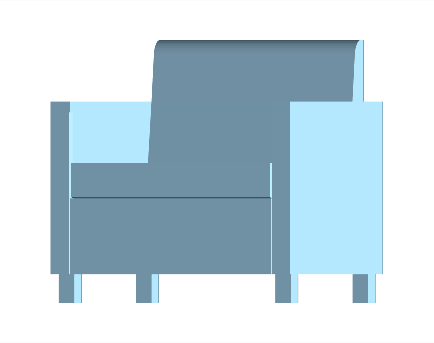}
\end{subfigure}

\vspace{-2mm}
    \caption{\textbf{Shortest paths comprising $k{=}3$ edges.}
    We find reasonable intermediate shapes (second and third column), between the source (left) and the target shape (right),~i.e.~between a standing cat and a running dog (top row, Non-Rigid World dataset~\cite{bronstein2006efficient}), as well as between a rounded armchair and a square armchair (bottom row, ShapeNet~\cite{shapenet2015}).
    }
        \label{fig:sp_k_3}
\end{figure*}

%% file: figures_inputs/interm_tosca.tex
\newcommand{\includeTosca}[1]{
        \includegraphics[width= 0.28\linewidth]{figures/intermediate_shapes_tosca/shape_tosca_mesh_#1.png}
}
\newcommand{\includeToscaRow}[1]{
\includeTosca{#1_1} &\includeTosca{#1_2} &\includeTosca{#1_3}
}

\begin{figure}[htb]
    \centering
    \resizebox{\linewidth}{!}{%
        \begin{tabular}{c c c}
             Source &  Intermediate & Target\\
             \toprule
            \includeToscaRow{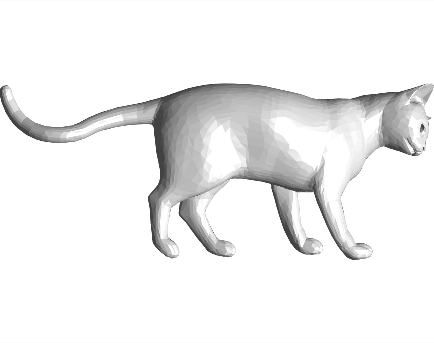} \\
            \toprule
            \includeToscaRow{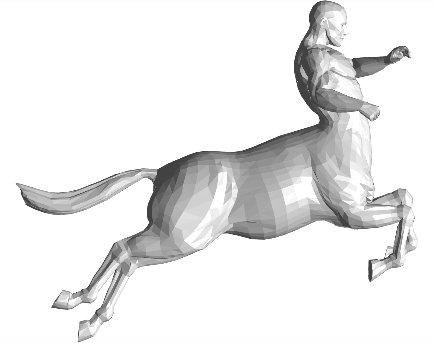} \\
            \toprule
            \includeToscaRow{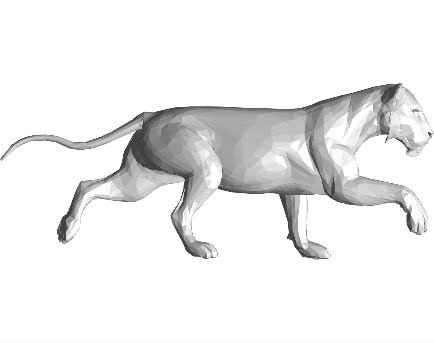} 
        \end{tabular}
    }
        \vspace{-5mm}
    \caption{\textbf{Intermediate shape between two different shapes.} By finding the shortest path from a source to a target shape, our algorithm is able to find a reasonable intermediate shape between a standing cat and an attacking 
    lion, a jumping centaur and a running lion, as well as a jumping dog and a standing lion (Non-Rigid World dataset~\cite{bronstein2006efficient}). }
    \label{fig:intermediateTosca}
\end{figure}

%% file: figures_inputs/chair_good.tex
\newcommand{\includeChair}[3]{\begin{subfigure}[b]{0.3\linewidth} \includegraphics[width= 0.95\linewidth]{figures/chair_#3/shape_chair_mesh_#1_#2.png}
\end{subfigure}}

\newcommand{\includeChairUp}[1]{\begin{subfigure}[b]{0.45\linewidth} \includegraphics[width= 0.9\linewidth]{figures/chair_upside/chair_#1_bad_up.png}
\end{subfigure}}

\newcommand{\rulesep}{\unskip\ \vrule\ }

\newcommand{\includeTransitionChair}[2]{\includeChair{#1}{1}{#2} \includeChair{#1}{2}{#2} \includeChair{#1}{3}{#2}}

\begin{figure*}
    \centering
    \begin{subfigure}[b]{0.45\linewidth}
        \includeTransitionChair{1}{good}
        \includeTransitionChair{2}{good}
        \includeTransitionChair{3}{good}
        \end{subfigure}
        \rulesep
\begin{subfigure}[b]{0.45\linewidth}
        \includeTransitionChair{4}{good}
        \includeTransitionChair{5}{good}
        \includeTransitionChair{6}{good}
         \end{subfigure}
         \vspace{-2mm}
    \caption{\textbf{Intermediate shapes for different topologies.} Our method can find reasonable intermediate shapes between various types of chairs from  ShapeNet~\cite{shapenet2015}, even if their topology varies drastically. }
    \label{fig:goodChairs}
\end{figure*}

%% file: figures_inputs/morphing.tex
\begin{figure*}[h!]
    \centering
    \includegraphics[width= 1\linewidth]{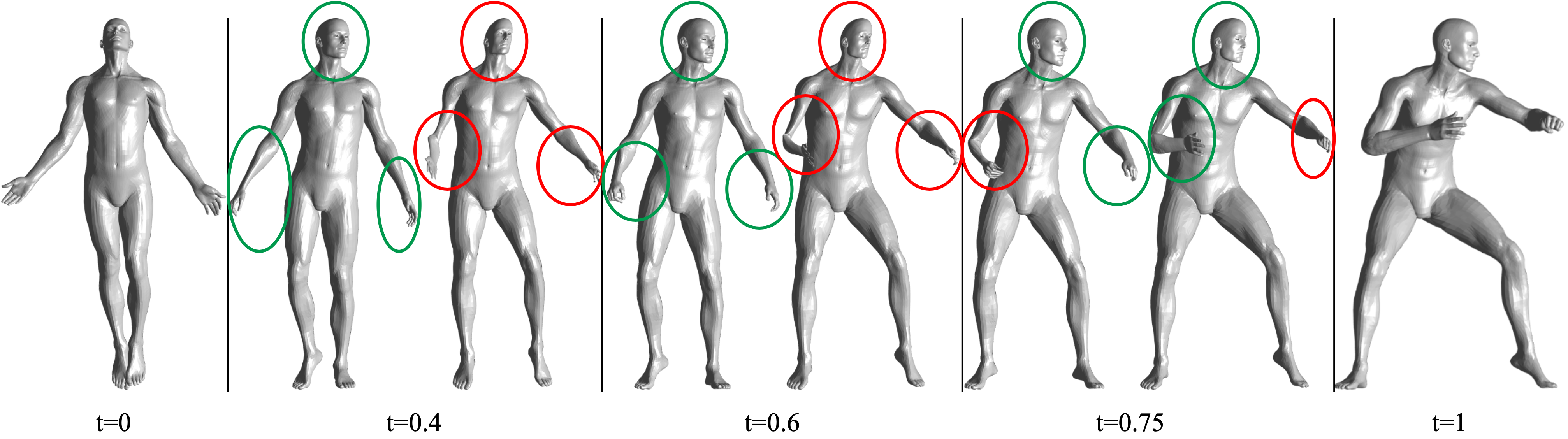}
    \vspace{-6mm}
    \caption{\textbf{Morphing the `Michael' shape from the  TOSCA dataset~\cite{bronstein2008numerical}.} Frames of the shape morphing process at t=0 (source), t=0.4, t=0.6, t=0.75 and t=1 (target) are shown. 
    For each time step, the left shape shows our morphing, whereas the right shape shows the morphing via naive linear interpolation. Overall, our morphing leads to significantly fewer artefacts.}
    \label{fig:shapeMorph}
\end{figure*}

%% file: supplementary_material.tex
\section{Appendix}
In this document we present additional results and an in-depth explanation of our graph generation.

\subsection{Additional Results}
\textbf{Shape morphing.}
We use the two settings ('Michael' and cat) from the TOSCA dataset \cite{bronstein2008numerical} to show exemplary shape morphing results using an intermediate shape found using our matrix-valued shortest path paradigm. In Fig.~\ref{fig:morphing}, we show the intermediate shapes between the source and target shape that were used to compute our morphing. 
For the 'Michael' shape (see Fig.~\ref{fig:shapeMorph}, as well as the animation in Fig.~\ref{fig:shape_morphing_animation}), we place the intermediate shape at t=0.5. In this case, we can observe a smooth morphing result.  

For the cat (see Fig.~\ref{fig:cat_morphing}, as well as the animation in Fig.~\ref{fig:shape_morphing_animation}), the intermediate shape is placed at t=0.8, because its pose lies much closer to the target shape than to the source shape. 
In this case,
there is no intermediate shape in the dataset that fits to t=0.5, so that particularly during the first part of the morphing we can observe some minor artefacts using our method, whereas naive linear interpolation leads to more drastic artefacts (see the gif animation). This emphasises that shape morphing is highly data-dependent, in the sense that the denser the different pose changes are covered in a dataset, the better the morphing results.
\newcommand{\includeMichael}[1]{
        \includegraphics[height= 0.42\linewidth]{figures/morphing/michael_#1.png}
}

\newcommand{\includeCat}[2]{
        \includegraphics[width= #2\linewidth]{figures/morphing/cat_#1.png}
}
\begin{figure}[h!]
    \centering
\begin{tabular}{c c c}
             Source &  Intermediate & Target\\
             \toprule
             \includeMichael{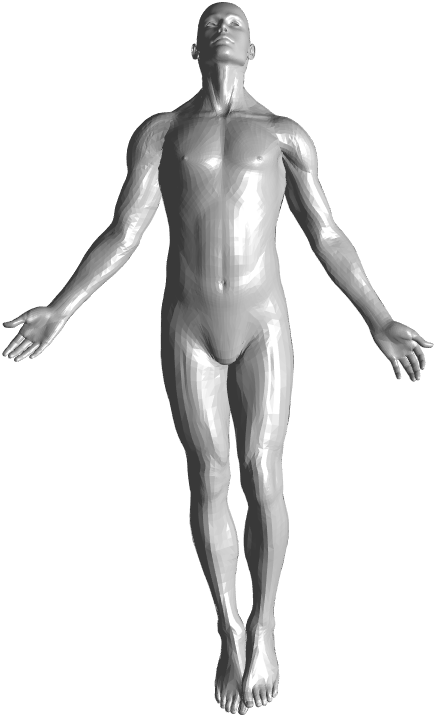} & \includeMichael{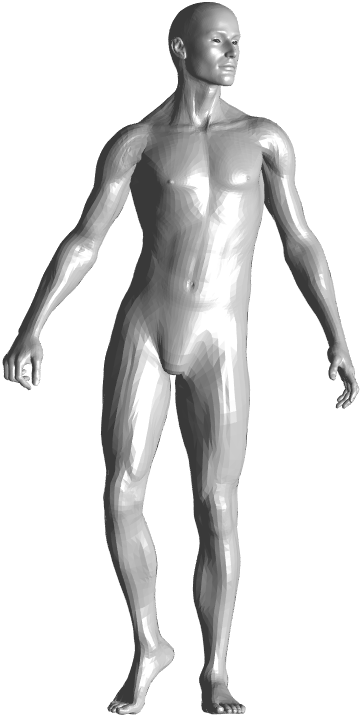}  & \includeMichael{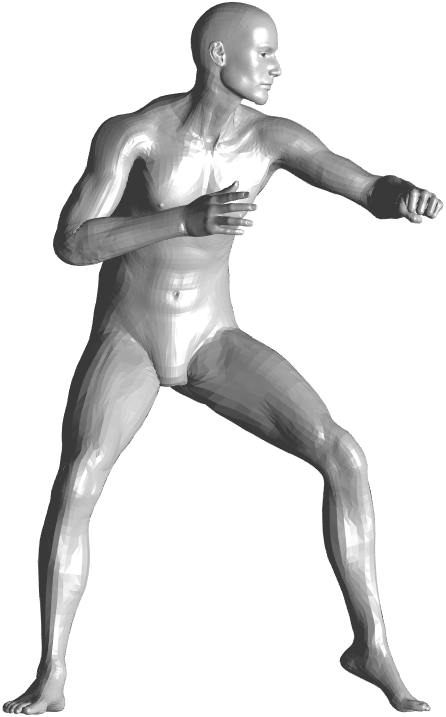} \\
            \toprule \includeCat{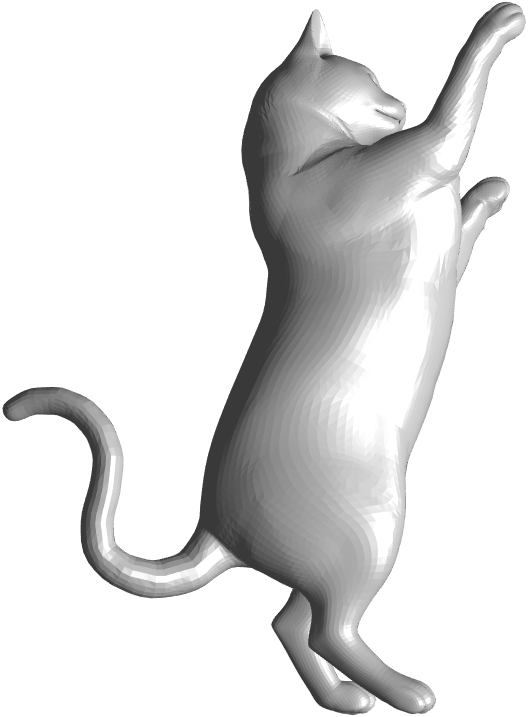}{0.15} & \includeCat{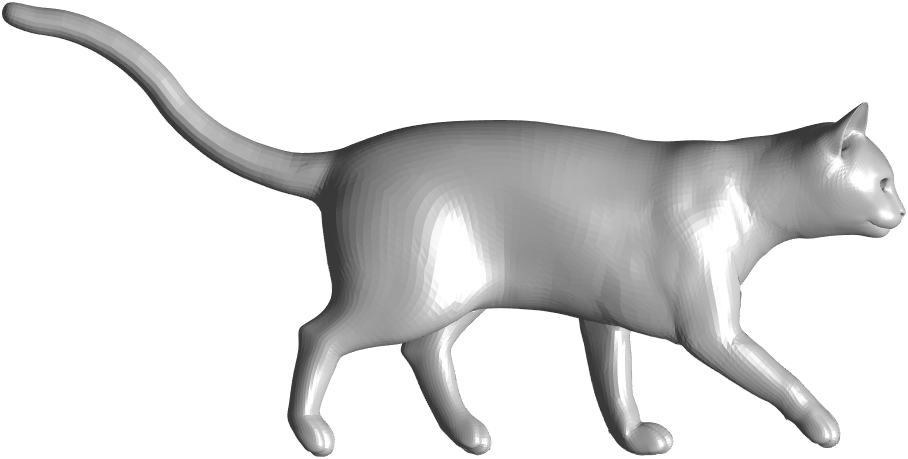}{0.28}  & \includeCat{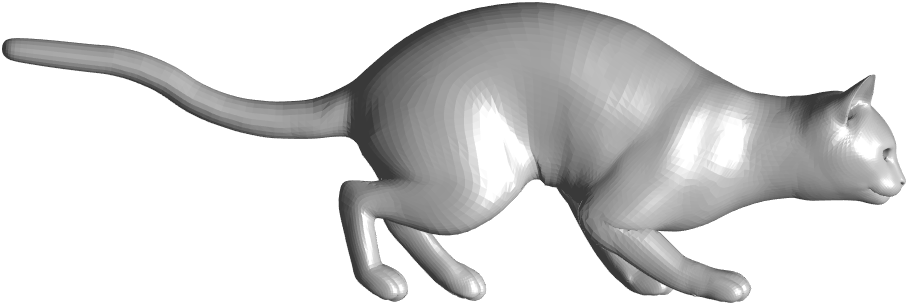}{0.28}
             
\end{tabular}
    \caption{\textbf{Intermediate shapes used for our shape morphing.} We use two example shapes of the TOSCA dataset \cite{bronstein2008numerical} (`Michael' and cat) to morph the source (left) over the intermediate (middle) to the target shape (right).}
    \label{fig:morphing}
\end{figure}
 
\begin{figure*}[h]
    \centering
    \includegraphics[width= 1\linewidth]{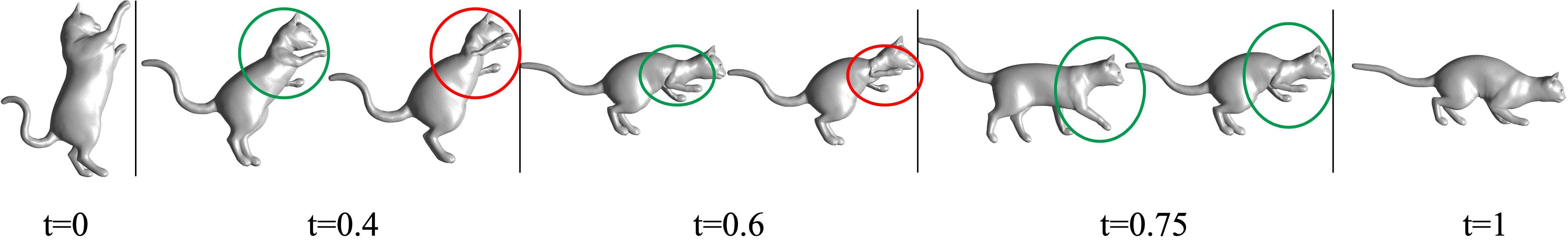}
    \caption{\textbf{Morphing the cat shape from the  TOSCA dataset~\cite{bronstein2008numerical}.} Frames of the shape morphing process at t=0 (source), t=0.4, t=0.6, t=0.75 and t=1 (target) are shown. For each time step, the left shape shows our morphing, whereas the right shape shows the morphing via naive linear interpolation. We can see that our morphing gives more plausible results compared to naive linear interpolation.}
    \label{fig:cat_morphing}
\end{figure*}

\begin{figure*}
    \centering
    \begin{tabular}{cc}
    \animategraphics[autoplay,loop, width=.48\linewidth]{50}{animations/michael/tosca_michael_new4-}{1}{100} &  \animategraphics[autoplay,loop, width=.48\linewidth]{50}{animations/cat/tosca_cat1-}{1}{100}
    \end{tabular}
    \caption{Animation of shape morphing of the 'Michael' and cat shapes from the TOSCA dataset~\cite{bronstein2008numerical}. \\
    (supported PDF reader required to play animation)}
    \label{fig:shape_morphing_animation}
\end{figure*}

\input{figures_inputs/chair_bad}
\input{figures_inputs/victoria_failures}

\textbf{Failure cases.}
Our edge matrices are obtained in terms of the similarity between  clusters extracted from two 3D shapes. While in many cases doing so leads to  sensible results,
there are also some failure cases. We illustrate one such case in Fig.~\ref{fig:badChairs}, where the found clusters  are inconsistent. In turn, the respective edge matrices are not very meaningful, so that an implausible intermediate shape is obtained. We emphasise that this is neither a limitation of our shortest path formalism, nor of our shortest path algorithm, but rather due to a suboptimal feature engineering.

Our edge matrices inherit rotation-invariance from the SHOT features.
Although rotation-invariance is a desirable property in many cases, in the context of shape morphing it may lead to undesirable results. For example,
the intermediate shape between two standing 'Victoria' shapes of the TOSCA dataset leads to a shape in lying position, as shown in Fig.~\ref{fig:failure_victoria} -- although in this case the intrinsic shape properties are indeed `intermediate', when considering the extrinsic shape,~e.g.~as it would be used for shape morphing, the results appear unnatural. Again, this is not a limitation of our method, but rather due to the way how the edge matrices are constructed.

\subsection{Details on Graph Generation}
\begin{figure*}[t!]
\includegraphics[width=1\linewidth]{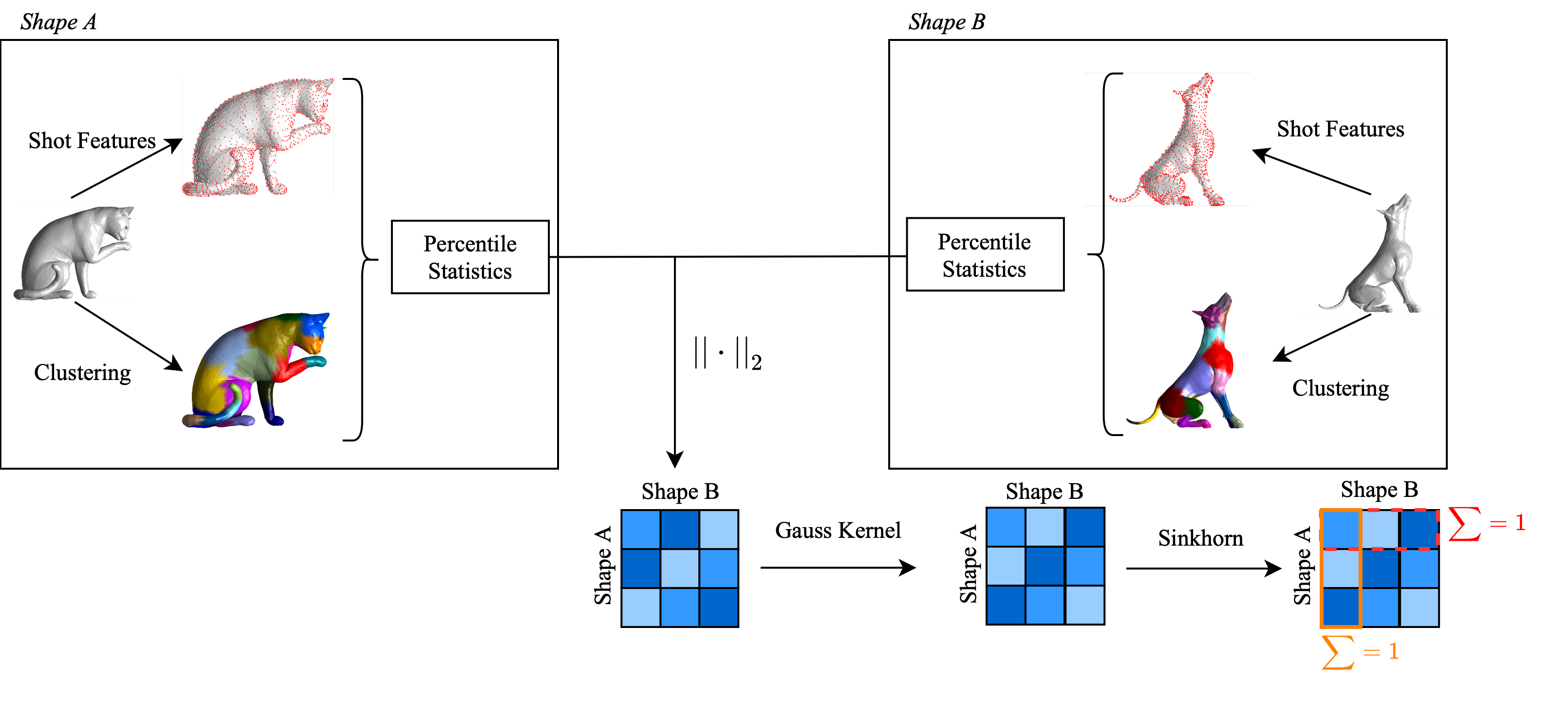}
\caption{\textbf{Probabilistic multi-matching graph construction.} First, we compute SHOT features on the 3D shapes \cite{tombari2010unique}, and perform a k-means clustering \cite{lloyd1982least}. We then summarise the distribution of the SHOT features in the clusters in terms of percentile statistics %
. We use the $\ell_2$-norm to compute differences between the per-cluster percentile statistics, so that we obtain a distance matrix of size $n \times n$.  By applying a Gaussian kernel, distances are translated into similarity scores. Eventually,  the Sinkhorn matrix scaling algorithm \cite{sinkhorn1967concerning}  is used to obtain doubly-stochastic matrices that can be interpreted as probabilities of matching pairs of clusters between shapes.
}
\label{fig:dataPreprocessing}
\end{figure*}
 
Next, we explain how we construct our probabilistic multi-matching graph for a given collection $\mathcal{S}$ of 3D shapes, where each shape in $\mathcal{S}$ represents a graph node. For each pair of shapes we compute a probabilistic correspondence matrix of size $n \times n$.
To obtain them, we first  compute SHOT features for all vertices of each 3D shape \cite{tombari2010unique}. 
Subsequently, in order to allow for shapes with varying number of vertices, we  cluster all vertices of one shape into $n$ clusters 
using the k-means algorithm~\cite{lloyd1982least}, where we choose $n=28$ in all experiments. 
Then,
we describe the SHOT feature distribution of all vertices within a cluster in terms of percentile statistics, which then serves as a cluster-specific feature descriptor that summarises the characteristics of each cluster. 
For each cluster $i$, and each of the shapes $x$ and $y$,  we obtain the percentile statistics matrices $P^{(x)}_i$ and $P^{(y)}_i$  of size $p \times f$, where $p$ is the number of considered percentiles and $f=352$ is the dimension of the SHOT feature descriptor.
For every pair of clusters between shape $x$ and $y$, we consider the Frobenius norm of the difference between the per-cluster percentile statistics. To this end, we define the distance between the $i$-th cluster of shape $x$ and the $j$-th cluster of shape $y$ as
\begin{equation}
    d_{ij} = \lVert P^{(x)}_i - P^{(y)}_j \lVert_F.
    \label{percentileCal}
\end{equation}
In order to transform the distances $d_{ij}$ to similarity scores, we apply a Gaussian kernel,~i.e.~
\begin{equation}
    m_{ij} = \exp(-\frac{d_{ij}^2}{\sigma^2}).
    \label{gauss}
\end{equation}
Eventually,
we use the Sinkhorn matrix scaling algorithm \cite{sinkhorn1967concerning} in order to obtain the doubly-stochastic matrix $M_{xy}$.
Here, each entry at position $(i,j)$ of the matrix $M_{xy}$ can be interpreted as the probability to match the $i$-th cluster of shape $x$ to the $j$-th cluster of shape $y$. The choice of the parameters $p$ and $\sigma$ is shown in Table~\ref{table:params}, and the entire process of the graph generation is outlined in Fig.~\ref{fig:dataPreprocessing}.
\begin{table}[]
\centering
\begin{tabular}{ccc}
    \toprule
    \textbf{Dataset} & \#percentiles $p$ & standard deviation $\sigma$\\
    \midrule
     Non-Rigid World & 300 & 2\\
     TOSCA 'Michael' & 3000 & 4.5\\
     TOSCA 'Victoria' & 150 & 0.7\\
     TOSCA Cat & 150 & 1\\
     SMAL & 300 & 3\\
     ShapeNet Chairs & 1000 & 2\\
     \bottomrule
\end{tabular}
\caption{Overview of chosen parameters.}\label{table:params}
\end{table}

%% file: figures_inputs/chair_bad.tex
\begin{figure}
    \centering
        \includeTransitionChair{bad}{bad}
        \includeChairUp{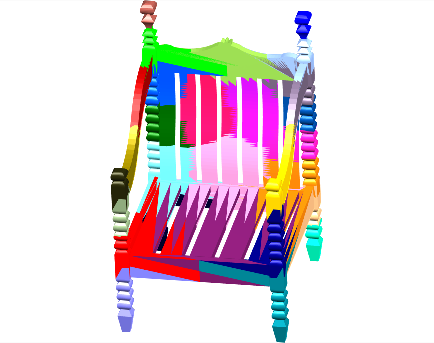}
        \includeChairUp{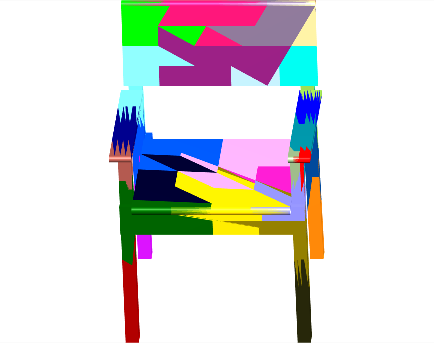}
    \caption{\textbf{Inconsistent clusters lead to unreliable intermediate shapes.} Using the chairs of ShapeNet~\cite{shapenet2015}, some obtained intermediate shapes are not plausible (top row). This is not a limitation of our method per-se, but stems from inconsistent clusters (bottom row).}
    \label{fig:badChairs}
\end{figure}

%% file: figures_inputs/victoria_failures.tex
\begin{figure}[ht]
    \centering
 \includegraphics[width= 1\linewidth]{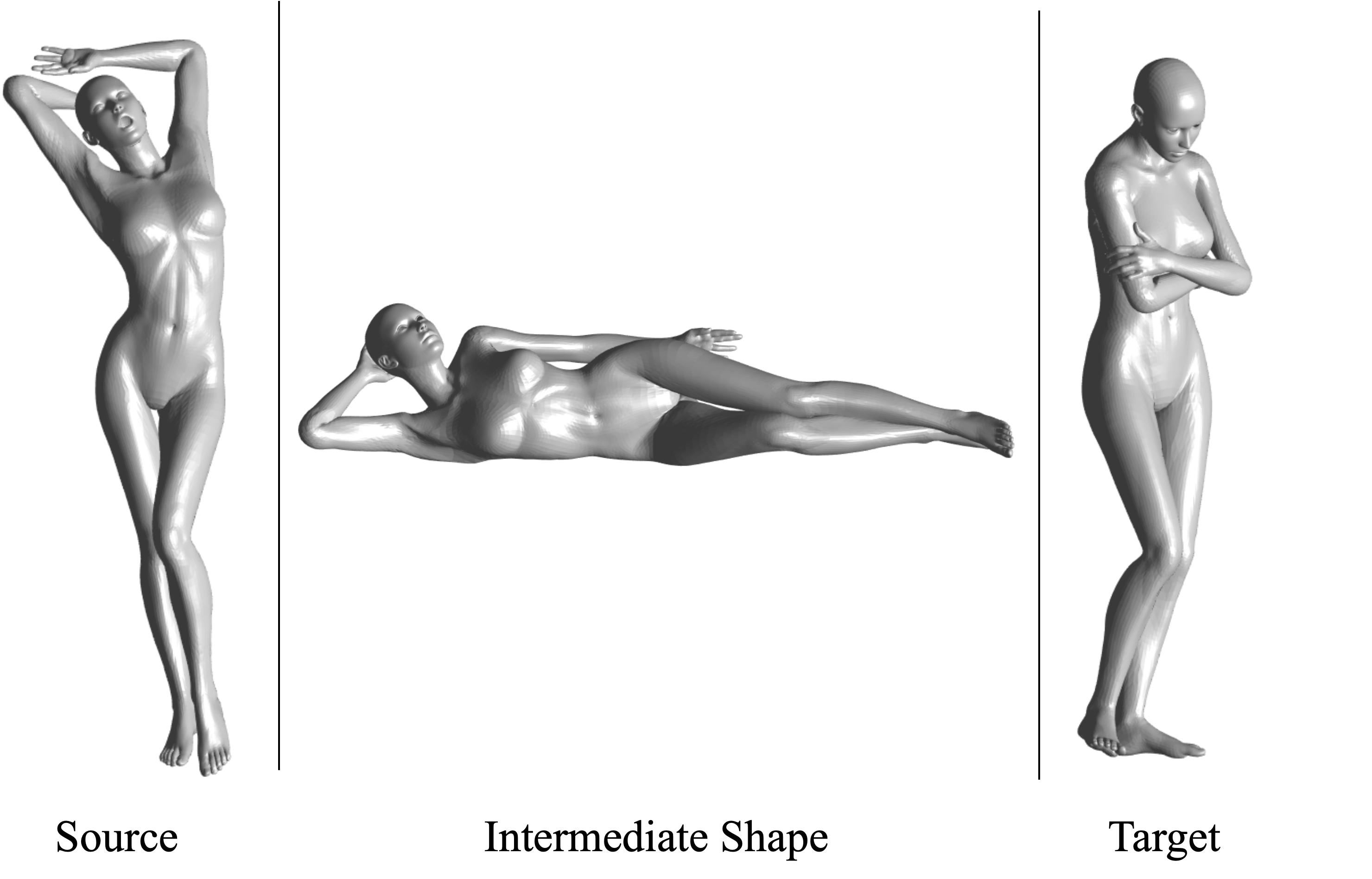}
    \caption{\textbf{Rotation invariance may lead to undesirable  intermediate shapes.} The intermediate shape of two standing 'Victoria' shapes of the TOSCA dataset  \cite{bronstein2008numerical} results in a lying shape that is unsuitable for shape morphing. This can be avoided by considering rotation-variant features for constructing the edge matrices.} 
    \label{fig:failure_victoria}
\end{figure}

%% file: main_3dv.bbl
\begin{thebibliography}{10}\itemsep=-1pt

\bibitem{agarwal2011building}
Sameer Agarwal, Yasutaka Furukawa, Noah Snavely, Ian Simon, Brian Curless,
  Steven~M Seitz, and Richard Szeliski.
\newblock Building rome in a day.
\newblock {\em Communications of the ACM}, 54(10), 2011.

\bibitem{aleksandrov2005determining}
Lyudmil Aleksandrov, Anil Maheshwari, and J-R Sack.
\newblock Determining approximate shortest paths on weighted polyhedral
  surfaces.
\newblock {\em Journal of the ACM (JACM)}, 52(1), 2005.

\bibitem{arrigoni2020synchronization}
Federica Arrigoni and Andrea Fusiello.
\newblock Synchronization problems in computer vision with closed-form
  solutions.
\newblock {\em International Journal of Computer Vision}, 128(1), 2020.

\bibitem{arrigoni2016spectral}
Federica Arrigoni, Beatrice Rossi, and Andrea Fusiello.
\newblock Spectral synchronization of multiple views in se(3).
\newblock {\em SIAM Journal on Imaging Sciences}, 9(4), 2016.

\bibitem{bellman1958routing}
Richard Bellman.
\newblock On a routing problem.
\newblock {\em Quarterly of applied mathematics}, 16(1), 1958.

\bibitem{berclaz2011multiple}
Jerome Berclaz, Francois Fleuret, Engin Turetken, and Pascal Fua.
\newblock Multiple object tracking using k-shortest paths optimization.
\newblock {\em IEEE transactions on pattern analysis and machine intelligence},
  33(9), 2011.

\bibitem{bernard2015solution}
Florian Bernard, Johan Thunberg, Peter Gemmar, Frank Hertel, Andreas Husch, and
  Jorge Goncalves.
\newblock A solution for multi-alignment by transformation synchronisation.
\newblock In {\em CVPR}, 2015.

\bibitem{bernard2019synchronisation}
Florian Bernard, Johan Thunberg, Jorge Goncalves, and Christian Theobalt.
\newblock Synchronisation of partial multi-matchings via non-negative
  factorisations.
\newblock {\em Pattern Recognition}, 92, 2019.

\bibitem{Bernard_2019_ICCV}
Florian Bernard, Johan Thunberg, Paul Swoboda, and Christian Theobalt.
\newblock Hippi: Higher-order projected power iterations for scalable
  multi-matching.
\newblock In {\em ICCV}, 2019.

\bibitem{besl1992method}
Paul~J Besl and Neil~D McKay.
\newblock Method for registration of 3-d shapes.
\newblock In {\em Sensor fusion IV: control paradigms and data structures},
  volume 1611, 1992.

\bibitem{birkhoff1946three}
Garrett Birkhoff.
\newblock Three observations on linear algebra.
\newblock {\em Univ. Nac. Tacuman, Rev. Ser. A}, 5, 1946.

\bibitem{bronstein2006efficient}
Alexander~M Bronstein, Michael~M Bronstein, and Ron Kimmel.
\newblock Efficient computation of isometry-invariant distances between
  surfaces.
\newblock {\em SIAM Journal on Scientific Computing}, 28(5), 2006.

\bibitem{bronstein2008numerical}
Alexander~M Bronstein, Michael~M Bronstein, and Ron Kimmel.
\newblock {\em Numerical geometry of non-rigid shapes}.
\newblock 2008.

\bibitem{shapenet2015}
Angel~X. Chang, Thomas Funkhouser, Leonidas Guibas, Pat Hanrahan, Qixing Huang,
  Zimo Li, Silvio Savarese, Manolis Savva, Shuran Song, Hao Su, Jianxiong Xiao,
  Li Yi, and Fisher Yu.
\newblock {ShapeNet: An Information-Rich 3D Model Repository}.
\newblock Technical Report arXiv:1512.03012, 2015.

\bibitem{chen1990shortest}
Jindong Chen and Yijie Han.
\newblock Shortest paths on a polyhedron.
\newblock In {\em Proceedings of the sixth annual symposium on Computational
  geometry}, 1990.

\bibitem{chen1992object}
Yang Chen and G{\'e}rard Medioni.
\newblock Object modelling by registration of multiple range images.
\newblock {\em Image and vision computing}, 10(3), 1992.

\bibitem{cover1999elements}
Thomas~M Cover.
\newblock {\em Elements of information theory}.
\newblock 1999.

\bibitem{dijkstra1959note}
Edsger~W Dijkstra et~al.
\newblock A note on two problems in connexion with graphs.
\newblock {\em Numerische mathematik}, 1(1), 1959.

\bibitem{eisenberger2020hamiltonian}
Marvin Eisenberger and Daniel Cremers.
\newblock Hamiltonian dynamics for real-world shape interpolation.
\newblock In {\em ECCV}, 2020.

\bibitem{eisenberger2019divergence}
Marvin Eisenberger, Zorah L{\"a}hner, and Daniel Cremers.
\newblock Divergence-free shape correspondence by deformation.
\newblock In {\em Computer Graphics Forum}, volume~38, 2019.

\bibitem{farahat1966semigroup}
HK Farahat.
\newblock The semigroup of doubly-stochastic matrices.
\newblock {\em Glasgow Mathematical Journal}, 7(4), 1966.

\bibitem{floyd1962algorithm}
Robert~W Floyd.
\newblock Algorithm 97: shortest path.
\newblock {\em Communications of the ACM}, 5(6), 1962.

\bibitem{ford1956network}
Lester~R Ford~Jr.
\newblock Network flow theory.
\newblock Technical report, Rand Corp Santa Monica Ca, 1956.

\bibitem{gojcic2020learning}
Zan Gojcic, Caifa Zhou, Jan~D Wegner, Leonidas~J Guibas, and Tolga Birdal.
\newblock Learning multiview 3d point cloud registration.
\newblock In {\em CVPR}, 2020.

\bibitem{grady2008minimal}
Leo Grady.
\newblock Minimal surfaces extend shortest path segmentation methods to 3d.
\newblock {\em IEEE Transactions on Pattern Analysis and Machine Intelligence},
  32(2), 2008.

\bibitem{hart1968formal}
Peter~E Hart, Nils~J Nilsson, and Bertram Raphael.
\newblock A formal basis for the heuristic determination of minimum cost paths.
\newblock {\em IEEE transactions on Systems Science and Cybernetics}, 4(2),
  1968.

\bibitem{heeren2012time}
Behrend Heeren, Martin Rumpf, Max Wardetzky, and Benedikt Wirth.
\newblock Time-discrete geodesics in the space of shells.
\newblock In {\em Computer Graphics Forum}, volume~31, 2012.

\bibitem{hilaga2001topology}
Masaki Hilaga, Yoshihisa Shinagawa, Taku Kohmura, and Tosiyasu~L Kunii.
\newblock Topology matching for fully automatic similarity estimation of 3d
  shapes.
\newblock In {\em Proceedings of the 28th annual conference on Computer
  graphics and interactive techniques}, 2001.

\bibitem{huang2021multibodysync}
Jiahui Huang, He Wang, Tolga Birdal, Minhyuk Sung, Federica Arrigoni, Shi-Min
  Hu, and Leonidas~J Guibas.
\newblock Multibodysync: Multi-body segmentation and motion estimation via 3d
  scan synchronization.
\newblock In {\em CVPR}, 2021.

\bibitem{huang2014functional}
Qixing Huang, Fan Wang, and Leonidas Guibas.
\newblock Functional map networks for analyzing and exploring large shape
  collections.
\newblock {\em TOG}, 33(4), 2014.

\bibitem{huang2019learning}
Xiangru Huang, Zhenxiao Liang, Xiaowei Zhou, Yao Xie, Leonidas~J Guibas, and
  Qixing Huang.
\newblock Learning transformation synchronization.
\newblock In {\em CVPR}, 2019.

\bibitem{jain1986algorithm}
VK Jain and Krishna Gopal.
\newblock An algorithm for determining the most reliable path of a network.
\newblock {\em Microelectronics Reliability}, 26(5), 1986.

\bibitem{jiang2007linear}
Hao Jiang, Sidney Fels, and James~J Little.
\newblock A linear programming approach for multiple object tracking.
\newblock In {\em CVPR}, 2007.

\bibitem{kaneva2000implementation}
Biliana Kaneva and Joseph O'Rourke.
\newblock {\em An implementation of Chen \& Han's shortest paths algorithm}.
\newblock PhD thesis, Smith College, Northampton, Mass., 2000.

\bibitem{kilian2007geometric}
Martin Kilian, Niloy~J Mitra, and Helmut Pottmann.
\newblock Geometric modeling in shape space.
\newblock In {\em ACM SIGGRAPH 2007 papers}. 2007.

\bibitem{kimmel1998computing}
Ron Kimmel and James~A Sethian.
\newblock Computing geodesic paths on manifolds.
\newblock {\em Proceedings of the national academy of Sciences}, 95(15), 1998.

\bibitem{kurillo2008wide}
Gregorij Kurillo, Zeyu Li, and Ruzena Bajcsy.
\newblock Wide-area external multi-camera calibration using vision graphs and
  virtual calibration object.
\newblock In {\em 2008 Second ACM/IEEE International Conference on Distributed
  Smart Cameras}, 2008.

\bibitem{laehner2016efficient}
Zorah Laehner, Emanuele Rodola, Frank~R Schmidt, Michael~M Bronstein, and
  Daniel Cremers.
\newblock Efficient globally optimal 2d-to-3d deformable shape matching.
\newblock In {\em CVPR}, 2016.

\bibitem{lanthier1997approximating}
Mark Lanthier, Anil Maheshwari, and J{\"o}rg-R{\"u}diger Sack.
\newblock Approximating weighted shortest paths on polyhedral surfaces.
\newblock In {\em Proceedings of the thirteenth annual symposium on
  Computational geometry}, 1997.

\bibitem{lloyd1982least}
Stuart Lloyd.
\newblock Least squares quantization in pcm.
\newblock {\em IEEE transactions on information theory}, 28(2), 1982.

\bibitem{lu1997globally}
Feng Lu and Evangelos Milios.
\newblock Globally consistent range scan alignment for environment mapping.
\newblock {\em Autonomous robots}, 4(4), 1997.

\bibitem{martinec2007robust}
Daniel Martinec and Tomas Pajdla.
\newblock Robust rotation and translation estimation in multiview
  reconstruction.
\newblock In {\em CVPR}, 2007.

\bibitem{mitchell1987discrete}
Joseph~SB Mitchell, David~M Mount, and Christos~H Papadimitriou.
\newblock The discrete geodesic problem.
\newblock {\em SIAM Journal on Computing}, 16(4), 1987.

\bibitem{pachauri2013solving}
Deepti Pachauri, Risi Kondor, and Vikas Singh.
\newblock Solving the multi-way matching problem by permutation
  synchronization.
\newblock In {\em NIPS}, 2013.

\bibitem{pan2008minimizing}
Feng Pan and David~P Morton.
\newblock Minimizing a stochastic maximum-reliability path.
\newblock {\em Networks: An International Journal}, 52(3), 2008.

\bibitem{petrovic1979two}
Radivoj Petrovic and Slobodan Jovanovic.
\newblock Two algorithms for determining the most reliable path of a network.
\newblock {\em IEEE Transactions on Reliability}, 28(2), 1979.

\bibitem{purkait2020neurora}
Pulak Purkait, Tat-Jun Chin, and Ian Reid.
\newblock Neurora: Neural robust rotation averaging.
\newblock In {\em ECCV}, 2020.

\bibitem{roosta1982routing}
Mohammad Roosta.
\newblock Routing through a network with maximum reliability.
\newblock {\em Journal of Mathematical Analysis and Applications}, 88(2), 1982.

\bibitem{schoenemann2009combinatorial}
Thomas Schoenemann and Daniel Cremers.
\newblock A combinatorial solution for model-based image segmentation and
  real-time tracking.
\newblock {\em IEEE Transactions on Pattern Analysis and Machine Intelligence},
  32(7), 2009.

\bibitem{sharp2002multiview}
Gregory~C Sharp, Sang~W Lee, and David~K Wehe.
\newblock Multiview registration of 3d scenes by minimizing error between
  coordinate frames.
\newblock In {\em ECCV}, 2002.

\bibitem{sinkhorn1967concerning}
Richard Sinkhorn and Paul Knopp.
\newblock Concerning nonnegative matrices and doubly stochastic matrices.
\newblock {\em Pacific Journal of Mathematics}, 21(2), 1967.

\bibitem{snavely2008skeletal}
Noah Snavely, Steven~M Seitz, and Richard Szeliski.
\newblock Skeletal graphs for efficient structure from motion.
\newblock In {\em CVPR}, 2008.

\bibitem{souiai2015entropy}
Mohamed Souiai, Martin~R Oswald, Youngwook Kee, Junmo Kim, Marc Pollefeys, and
  Daniel Cremers.
\newblock Entropy minimization for convex relaxation approaches.
\newblock In {\em ICCV}, 2015.

\bibitem{strasdat2010scale}
Hauke Strasdat, J Montiel, and Andrew~J Davison.
\newblock Scale drift-aware large scale monocular slam.
\newblock {\em Robotics: Science and Systems VI}, 2(3), 2010.

\bibitem{sundar2003skeleton}
Hari Sundar, Deborah Silver, Nikhil Gagvani, and Sven Dickinson.
\newblock Skeleton based shape matching and retrieval.
\newblock In {\em 2003 Shape Modeling International.}, 2003.

\bibitem{thunberg2017distributed}
Johan Thunberg, Florian Bernard, and Jorge Goncalves.
\newblock Distributed methods for synchronization of orthogonal matrices over
  graphs.
\newblock {\em Automatica}, 80, 2017.

\bibitem{thunberg2017distributedCdc}
Johan Thunberg, Florian Bernard, and Jorge Gon{\c{c}}alves.
\newblock Distributed synchronization of euclidean transformations with
  guaranteed convergence.
\newblock In {\em Conference on Decision and Control (CDC)}, 2017.

\bibitem{tombari2010unique}
Federico Tombari, Samuele Salti, and Luigi Di~Stefano.
\newblock Unique signatures of histograms for local surface description.
\newblock In {\em ECCV}, 2010.

\bibitem{wang2018pixel2mesh}
Nanyang Wang, Yinda Zhang, Zhuwen Li, Yanwei Fu, Wei Liu, and Yu-Gang Jiang.
\newblock Pixel2mesh: Generating 3d mesh models from single rgb images.
\newblock In {\em ECCV}, 2018.

\bibitem{Zuffi:CVPR:2017}
Silvia Zuffi, Angjoo Kanazawa, David Jacobs, and Michael~J. Black.
\newblock {3D} menagerie: Modeling the {3D} shape and pose of animals.
\newblock In {\em CVPR}, 2017.

\end{thebibliography}
